\documentclass[letterpaper, journal]{IEEEtran} 
\IEEEoverridecommandlockouts

\usepackage[english]{babel}
\usepackage[T1]{fontenc}

\usepackage[nolist]{acronym}
\usepackage[ruled,vlined]{algorithm2e}
\usepackage[cmex10]{amsmath}
\usepackage{amssymb}
\usepackage{amsthm}
\usepackage{bm}
\usepackage{booktabs}
\usepackage{color}
\usepackage{float}
\usepackage{graphicx,subfigure}
\usepackage{hyperref}
\usepackage{import}
\usepackage{mathtools}
\usepackage{multirow}
\usepackage{outline}
\usepackage{paralist}
\usepackage{siunitx}
\usepackage{stmaryrd}
\usepackage{systeme}
\usepackage{units}
\usepackage{url}
\usepackage{tikz}
\usepackage{subfigure}
\usetikzlibrary{tikzmark}

\newcommand{\mnorm}[1]{\left\| #1 \right\|}

\newcommand{\mmat}[1]{\begin{bmatrix} #1 \end{bmatrix} }

\newcommand{\mpd} [2]{\frac{\partial #1}{\partial #2}} 

\newcommand{\figref}[1]{Figure~\ref{#1}}
\newcommand{\secref}[1]{Section~\ref{#1}}

\newtheorem{prop}{Proposition}

\title{Model-free online motion adaptation for energy-efficient flight of multicopters}
\author{Xiangyu Wu$^1$, Jun Zeng$^2$, Andrea Tagliabue$^3$, and Mark W. Mueller$^1$
\thanks{$^1$Authors are with the High Performance Robotics Laboratory (HiPeRLab) at the Department of Mechanical Engineering, UC Berkeley. {\tt\small \{wuxiangyu,mwm\}@berkeley.edu}}
\thanks{$^2$Author is with the Hybrid Robotics Group at the Department of Mechanical Engineering, UC Berkeley. {\tt\small zengjunsjtu@berkeley.edu}}
\thanks{$^3$Author is with the Aerospace Controls Laboratory at the Department of Aeronautics and Astronautics, MIT. {\tt\small atagliab@mit.edu }} 
}

\begin{document}
\maketitle

\begin{abstract}
Limited flight distance and time is a common problem for multicopters. 
We propose a method for finding the optimal speed and sideslip angle of a multicopter flying a given path to achieve either the longest flight distance or time.
Since flight speed and sideslip are often free variables in multicopter path planning, they can be changed without changing the mission.
The proposed method is based on a novel multivariable extremum seeking controller with adaptive step size, which is inspired by recent work from the machine learning community on stochastic optimization.
Our method 
(a) does not require a power consumption model of the vehicle, 
(b) is computationally efficient and runs on low-cost embedded computers in real-time, and 
(c) converges faster than the standard extremum seeking controller with constant step size.
We prove the stability of this approach and validate it through outdoor experiments.
The method is shown to converge with different payloads and in the presence of wind.
Compared to flying at the maximum achievable speed in the experiments with a uniformly selected random sideslip angle, flying at the optimal range speed and sideslip on average increases the flight range by 14.3\% without payload and 19.4\% with a box payload.
In addition, compared to hovering, flying at the optimal endurance speed and sideslip  increases the flight time by 7.5\% without payload and 14.4\% with a box payload.
A video can be found at \url{https://youtu.be/aLds8LVfogk}
\end{abstract}

\section{Introduction}
\label{sec:intro}

Multicopters are used in a wide range of applications such as aerial photography \cite{photography2}, transportation \cite{transportation2}, search and rescue \cite{search_and_rescue}, inspection \cite{inspection2}, and agriculture \cite{farming2}, thanks to their low cost, ease of control, and high maneuverability.
However, a primary limitation for current vehicles is their limited flight endurance and range \cite{karydis2017energetics}. 

One way to improve the limited flight range or endurance problem is through energy-efficient mechanical design.
For example, in \cite{TriangularQuad} a triangular quadcopter with one large rotor for lifting and three small rotors for control was proposed, which has the advantage of combining the energy efficiency of the large rotor and the fast control response of the small rotors.
In \cite{tilt_motor_drone}, the authors designed a quadcopter with slightly tilted motors which has a better control authority over the yaw. 
This results in a lower variance in motor forces for yaw control.
Because a motor's power is a convex function of its thrust, this design helps to reduce the total power consumption of the motors.
Hybrid quadcopters which are able to do both aerial and ground locomotion, were introduced in \cite{HyTAQ2} and \cite{shapeshifter}: when the vehicles operate in the ground locomotion mode on a flat ground, they only need to overcome the rolling resistance and use much less power compared to flying.
A hybrid power system for multicopters consisting of a lithium battery, a fuel cell, and a hydrogen tank was introduced in \cite{fuel_cell_drone}, which enables longer flight time compared to traditional battery-only power systems, thanks to the higher specific energy of hydrogen compared with the lithium battery.
In \cite{switchingBattery}, an in-flight battery switching system was proposed, which enables a small quadcopter to dock an additional battery to a large quadcopter and increases its flight time.


\begin{figure}[t]
	\centering
	\hspace{-2.5mm} \subfigure{\includegraphics[width=0.98\linewidth]{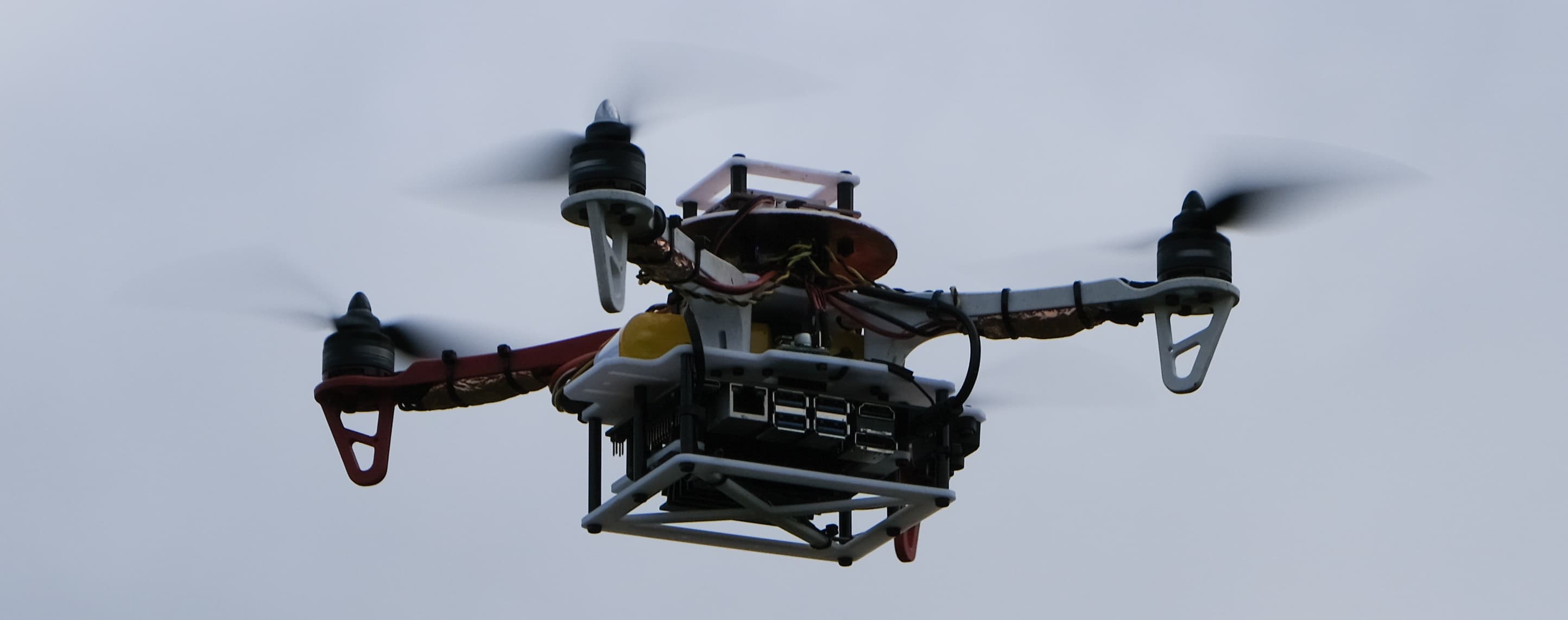}} \\[-0.8ex] 
	\subfigure{\includegraphics[width=0.98\linewidth]{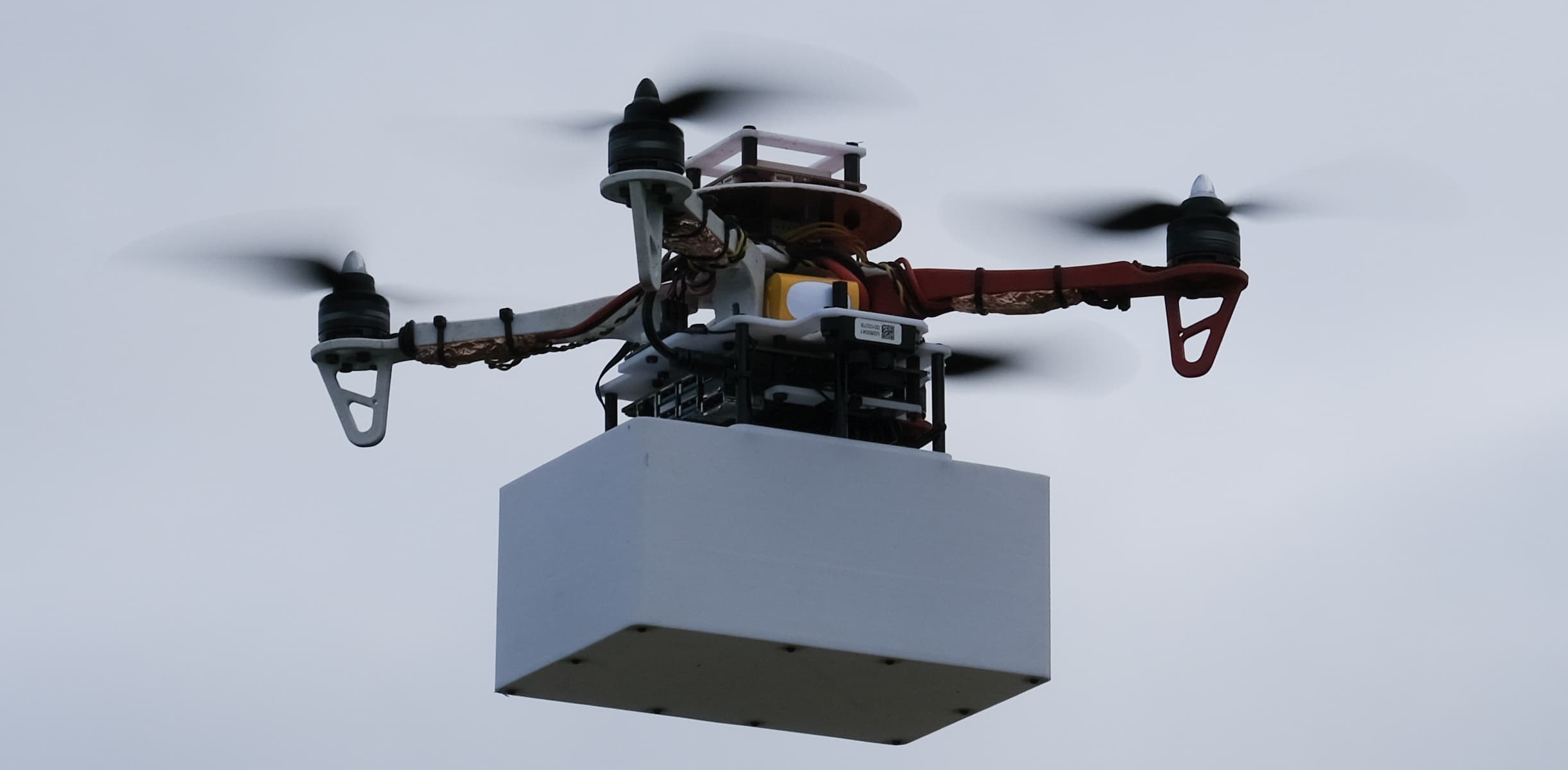}}
	\vspace{-1ex}
	\caption{
	A quadcopter with and without a box payload with unknown aerodynamics effects, as used in our experiments.
	}
	\label{fig:front}
\end{figure}

Another category of methods focus on developing algorithms to reduce the power consumption of existing multicopters.
By planning energy-efficient trajectories or by implementing energy-aware control algorithms, these approaches do not require design changes to existing hardware and are thus economical to deploy.
For example, in \cite{modelBased2} the authors proposed a method for finding the minimum-energy trajectory between a predefined initial and final state of a quadcopter, by solving an optimal control problem of the angular accelerations of the four propellers. 
This approach was extended in \cite{modelBased3}, where the fixed end-time trajectory optimization was extended to both free and fixed end-time solved with an indirect projected gradient algorithm to improve the numerical accuracy.
Simulation results were shown to validate the effectiveness of the methods in both papers.
In \cite{modelBased1}, the task of reaching a goal in a set of candidate goals while using the least amount of energy was investigated.
The energy-efficient path planning algorithm was based on the model predictive control and disturbance from wind was considered.
The authors showed that their method was able to reach the goal which required the least amount of energy in simulations and indoor experiments.
In \cite{modelBased4} and \cite{modelBased5}, the authors proposed energy-aware coverage path planning methods for photogrammetric sensing of large areas using multicopters.
The methods find the optimal speed along the coverage path to minimize the energy usage during the mission.
Outdoor experiments were conducted to validate their methods.
\par
A necessary condition for model-based methods to perform well is accurate power consumption modeling. 
Power consumption models of multicopters can be derived by analyzing their electric and aerodynamic properties.
For example, \cite{ampatis2014parametric} \cite{uavEnergyModeling} introduced power consumption models of the battery, electric speed controller and motor, and \cite[Chapter 5]{leishman2006principles} introduced the aerodynamic power consumption of the propeller based on the momentum theory. 
Besides, some researchers proposed data-driven models by selecting variables that affect the power consumption (e.g. the vehicle's speed and acceleration, wind speed, and payload weight) as inputs and finding their relationship to power consumption through experimental data \cite{modelBased4} \cite{dataDrivenModel}. 

However, there are often hard-to-model effects on the vehicle's power consumption, such as changes in vehicle components' performance (e.g. batteries and motors) due to aging and temperature changes.
In addition, the change in payload size, shape, or weight in applications such as package delivery and spraying (e.g., pesticides or fertilizer at farms) often requires reidentification of parameters in the power consumption model, which is time-consuming. 
The imperfections in the energy model could potentially be compensated using online data-driven methods. For example, in \cite{LBMPC} the authors used an Extended Kalman Filter and in \cite{DataDrivenMPC} the authors used Gaussian processes to estimate the correction terms in the vehicle's dynamics equations, which improved the control accuracy of the quadcopters. However, to the best of our knowledge, no such methods have been developed for the energy efficient flight of quadcopters yet, and their effectiveness and computational efficiency are thus still an open question.

The aforementioned difficulties in quadcopter energy consumption modeling motivates us to propose, to the best of our knowledge, the first model-free method for finding the flight speed and sideslip angle (i.e., angle between the forward direction of the vehicle and the relative wind) which achieve the longest flight time (endurance) or flight range given a predefined path.
The method is based on a novel multivariable extremum seeking controller and does not require power consumption models of the multicopter.
\par
Extremum seeking control is a model-free adaptive control technique for finding the local minimizer of a given, potentially time-varying, cost function by applying a persistently exciting periodic perturbation to a set of chosen inputs, and monitoring the corresponding output changes.
A survey of the development of this control method can be found at \cite{ESC_survey}.
It has applications in areas such as maximizing the energy generation of wind turbines \cite{wind_turbine} and photovoltaic power plants \cite{photovoltaic_power}, 
and maximizing the pressure rise in axial flow compressors \cite{ESC_application_compressor}.
Its applications in robotics can be found in a literature survey \cite{esc_robotics_applications}.
A common problem of extremum seeking controllers is their slow convergence speed, and we propose a novel multivariable extremum seeking controller with adaptive step size to improve it. 
In addition to the flight speed, it could also simultaneously find the optimal flight sideslip angle to achieve the longest flight range or endurance (time). 
\par


The major contributions of this paper are as follows:
\begin{enumerate}
    \item We present a model-free adaptive method to find the flight speed and sideslip angle of multicopters that achieve the longest flight range or endurance. 
    \item The method is based on a novel multivariable extremum seeking controller with adaptive step size, which is computationally efficient and converges faster than the standard extremum seeking.
    \item We give a stability proof for the proposed controller via averaging and singular perturbation analysis.
    \item We validate the effectiveness of the proposed method in extensive outdoor experiments. The experiments demonstrate the proposed method's faster convergence compared with the standard method, robustness to payloads and wind disturbances.
\end{enumerate}

\par
This is an evolved paper based on our prior work  \cite{previous_paper_1, previous_paper_2}.
In contrast to the prior work, this paper presents:
\begin{enumerate}
    \item A stability proof of the proposed multivariable extremum seeking controller with adaptive step size taking into account the vehicle's dynamics.
    \item Extensive outdoor experiments with practical real-world sensing instead of the  previous work's indoor experiments with a motion capture system for state estimation.
    \item Applications of extremum seeking to time optimal flight in addition to range optimal flight, by searching for the optimal endurance flight speed and sideslip angle.
    \item Experiments and discussion about the energy cost from the extremum seeking controller because of perturbation.
    \item Experiments and analysis about the proposed method's performance under wind disturbances.
\end{enumerate}


\label{sec:Introduction}

\section{Problem statement}
\label{sec:problem_statement}

\begin{figure*}[!ht]
	\begin{center}    
		\includegraphics[width=\linewidth]{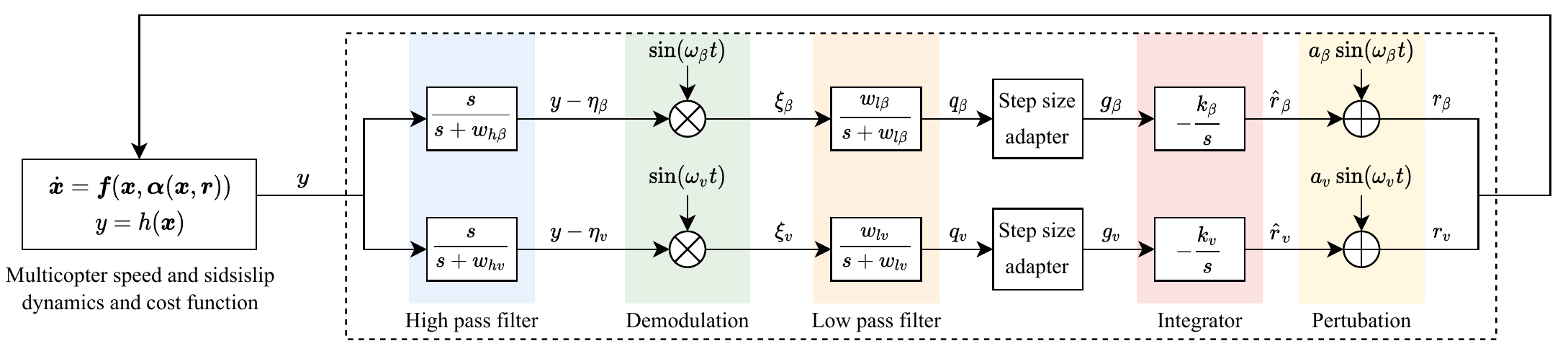}
	\end{center}
	\vspace{-1ex}
	\caption{Block diagram of the adaptive step size multivariable extremum seeking controller (in the dashed rectangle). 
	The goal of the controller is to find the optimal sideslip $r_\beta$ and $r_v$ to minimize the cost function $y = (h \circ l)(\bm{r})$. 
	The frequencies of the high pass and low pass filters are set, respectively, to $\omega_{{hv}}$ and $\omega_{{lv}}$ for speed, and $\omega_{{h\beta}}$ and $\omega_{{l\beta}}$ for sideslip. 
	The scalar $k_v$ and $k_\beta$ are related to the step size of the extremum seeking controller and both of them should be positive numbers to minimize the cost function. 
	The standard extremum seeking controller with sinusoidal perturbations does not have the step size adapter and the outputs of the low pass filters directly go to the integrator, while the remaining structure of the algorithm is exactly the same. The step size adapter is detailed in \secref{subsubsec:step_size_adapter}. 
	}
	\label{fig:escDiagram}
\end{figure*}

In this work, we propose a method to find the most  energy-efficient flight speed and sideslip angle to mitigate the common problem of the limited flight range and endurance of multicopters. 
\par

\begin{table}[!h]
	\def\arraystretch{1.2}
	\caption{Notations used in \secref{sec:method}.}
	\vspace{-1ex}
	\label{tab:symbols}
	\centering

\begin{tabular}{|l|l|}
\hline
Symbol & Meaning \\ \hline
$\bm{x} = [v,  \beta]^T$ & \begin{tabular}[c]{@{}l@{}}state variables related to energy-efficient flight,\\ $v$ is flight speed, $\beta$ is sideslip angle\end{tabular} \\ \hline
$\bm{r} = [r_v, r_\beta]^T$ & \begin{tabular}[c]{@{}l@{}}extremum seeking controller's outputs,\\ $r_v$ is ref. speed, $r_\beta$ is ref. sideslip\end{tabular} \\ \hline
$\hat{\bm{r}} = [\hat{r}_v, \hat{r}_\beta]^T$ & \begin{tabular}[c]{@{}l@{}}ref. speed and sideslip without perturbations\\ (outputs of the integrators)\end{tabular} \\ \hline
$\bm{r}^* = [r_v^*, r_\beta^*]^T$ & local minimum of the energy cost function \\ \hline
$\tilde{\bm{r}} = [\tilde{r}_v, \tilde{r}_\beta]^T$ & difference between $\bm{r}$ and the optimum $\bm{r}^{*}$ \\ \hline
$\bm{\alpha}(\bm{x}, \bm{r})$ & the control law for speed and sideslip \\ \hline
\begin{tabular}[c]{@{}l@{}}$\dot{\bm{x}} = $\\ $\bm{f} (\bm{x}, \bm{\alpha}(\bm{x}, \bm{r}))$\end{tabular} & closed-loop dynamics of speed and sideslip \\ \hline
$y = h(\bm{x})$ & $y$ is the energy cost, $h(\cdot)$ is the cost function \\ \hline
$\bm{l(r)}$ & equilibrium point of $\bm{x}$ \\ \hline
$\bm{p}(t)$, $\bm{d}(t)$ & perturbation signals, demodulation signals \\ \hline
$a_v$, $a_{\beta}$ & perturbation magnitude of ref. speed, sideslip \\ \hline
$\omega_v$, $\omega_{\beta}$ & perturbation frequency of ref. speed, sideslip \\ \hline
$\omega_{hv}$, $\omega_{h\beta}$ & high-pass filters' cutoff frequencies \\ \hline
$\omega_{lv}$, $\omega_{l\beta}$ & low-pass filters' cutoff frequencies \\ \hline
$\eta_v$, $\eta_\beta$ & \begin{tabular}[c]{@{}l@{}}low-frequency components of the cost function, \\ filtered out by the high pass filters\end{tabular} \\ \hline
$\tilde{\eta}_v$, $\tilde{\eta}_\beta$ & \begin{tabular}[c]{@{}l@{}}difference between $\eta_v$ and $\eta_\beta$\\ to optimal cost $(h \circ \bm{l})(\bm{r}^*)$, see \eqref{equ:pertubations}\end{tabular} \\ \hline
$\xi_v$, $\xi_\beta$ & \begin{tabular}[c]{@{}l@{}}product of high-frequency components of the cost\\ function with the demodulation signals\end{tabular} \\ \hline
$q_v$, $q_\beta$ & \begin{tabular}[c]{@{}l@{}}approximations of the cost function's gradient to \\ speed and sideslip\end{tabular} \\ \hline
$m_v$, $m_\beta$ & estimates of the second moments of  $q_v$, $q_\beta$ \\ \hline
$g_v$, $g_\beta$ & output of the step size adapters \\ \hline
$\epsilon$ & \begin{tabular}[c]{@{}l@{}}a small positive constant preventing \\ dividing by zero.\end{tabular} \\ \hline
$\gamma_v$, $\gamma_\beta$ & cut-off frequencies for low-pass filters of  $q_v^2$, $q_\beta^2$ \\ \hline
$k_v$, $k_\beta$ & \begin{tabular}[c]{@{}l@{}}positive constants related to the step-size \\ of gradient descent\end{tabular} \\ \hline
$\delta$, $\omega$ & \begin{tabular}[c]{@{}l@{}}small positive constants used in \\ the stability proof, see \eqref{equ:prime_symbols}\end{tabular} \\ \hline
\begin{tabular}[c]{@{}l@{}}$\omega_v^{'}$, $\omega_\beta^{'}$,  $\omega_{hv}^{'}$, \\ $\omega_{h\beta}^{'}$, $\omega_{lv}^{'}$, $\omega_{l\beta}^{'}$, \\ $k_v^{'}$, $k_\beta^{'}$, $\gamma_v^{'}$, $\gamma_\beta^{'}$\end{tabular} & \begin{tabular}[c]{@{}l@{}}constants used in the proof, related to \\ constants without prime superscript, see \eqref{equ:prime_symbols} \end{tabular} \\ \hline
$\tau = \omega t$ & a time scale used in proof \\ \hline
$\bm{\Bar{p}}(\tau) = \bm{p}(t/\omega)$ & representation of $\bm{p}$ function under time scale $\tau$ \\ \hline
$\Pi$ & \begin{tabular}[c]{@{}l@{}}the least common period of functions with \\ frequencies of  $\omega_v^{'}$ and $\omega_\beta^{'}$\end{tabular} \\ \hline
\end{tabular}
	
	\end{table}

We choose to optimize these two variables because they affect the vehicle's power consumption and are typically additional (redundant) degrees of freedom in a multicopter's flight, where the flight missions require the vehicle to track specified geometric paths.
Because the multicopter is usually not axisymmetric (especially when carrying payloads), flying with different sideslip angles affects the drag force faced by the vehicle and leads to different power consumption.
The sideslip angle can be changed by changing the yaw angle.
The flight speed also affects the power consumption of the vehicle: when the flight speed increases, the power consumption first decreases and then increases, which can be explained by momentum theory \cite[Chapter 2.14]{leishman2006principles}.
This predicts that the maximum flight endurance is achieved by flying at a suitable flight speed, rather than hovering.
\par
When our goal is to achieve the longest flight endurance (time), we want to minimize the consumed energy for a given time.
As a result, the cost function for the optimal endurance flight is defined as the instantaneous electric power ${p_e}$. 
When the goal is to achieve the longest flight range (distance), we want to minimize the energy consumed for a given distance. 
Thus, the cost function for the optimal range flight is the instantaneous electric power over speed ${p_e}/{v}$ (i.e. energy over distance), where $v$ denotes the speed of the vehicle.

\par
A model-free optimization method is preferable, which can handle hard-to-model effects (e.g., components aging and temperature change) and payload changes.
This motivates us to use an extremum seeking controller to find the optimal flight speed and sideslip angle.
The required inputs to the extremum seeking controller are the instantaneous energy cost and a user-defined geometric path. 
Its outputs are the vehicle's reference speed and sideslip angle commands, which are then used to convert the geometric path into a reference trajectory to be tracked by the low-level controllers.

\section{Model-free speed and sideslip adaptation}
\label{sec:method}
In this section, we introduce the novel multivariable extremum seeking controller with adaptive step size.
It is able to achieve faster convergence than the standard extremum seeking controller with a fixed step size, by taking a smaller step size when the estimated gradient has a large magnitude or variance and vice versa.
Vector variables and functions that map to vectors are written in boldface.
Notations in this section are summarized in Table \ref{tab:symbols}.
\subsection{Extremum seeking controller with adaptive step size}
\label{sec:esc}
A block diagram of the proposed adaptive-step-size, multivariable, extremum seeking controller is shown in \figref{fig:escDiagram}.
We define the state variables of the multicopter (relevant to our problem) as $\bm{x} = [v,  \beta]^T$, where $v$ and $\beta$ are the speed and sideslip of the vehicle, respectively.
The outputs of the extremum seeking controller are defined as $\bm{r} = [r_v, r_\beta]^T$, where $r_v$ is the reference flight speed and $r_\beta$ is the reference flight sideslip. 
We assume a smooth control law $\bm{\alpha}(\bm{x}, \bm{r})$, so that the closed-loop dynamics of the speed and sideslip are represented by 
\begin{align}
    \dot{\bm{x}} = \bm{f} (\bm{x}, \bm{\alpha}(\bm{x}, \bm{r})). \label{eq:plant_dynamics}
\end{align}
The cost function is represented by 
\begin{align}
    y = h(\bm{x}).
\end{align}
Like in \cite{esc_stability}, we make the following assumptions about the closed-loop vehicle dynamics and the cost function:
\par
\textbf{Assumption 1.} 
There exists a smooth function $\bm{l}: \mathbb{R}^2 \rightarrow \mathbb{R}^2$ such that $\bm{f} (\bm{x}, \bm{\alpha}(\bm{x}, \bm{r})) = 0$ if and only if $\bm{x} = \bm{l}(\bm{r})$.
\par
\textbf{Assumption 2.} 
For each reference input  $\bm{r}$, the controller ensures that the equilibrium $\bm{x} = \bm{l}(\bm{r})$ is locally exponentially stable uniformly in $\bm{r}$.
\par
Thus, we assume that we have a control law $\bm{\alpha}(\bm{x}, \bm{r})$, that can locally stabilize any of the equilibria that $\bm{r}$ may produce. 
\par
\textbf{Assumption 3.} The cost function (described in \secref{sec:problem_statement}) has a local minimum at $\bm{r}^* = [r_v^*, r_\beta^*]^T$, such that
\begin{equation}
    \triangledown (h \circ \bm{l})(\bm{r}^*) = 0, \quad \triangledown^2 (h \circ \bm{l})(\bm{r}^*) > 0. \label{equ:minimum-nominal-dynamics}
\end{equation}
    
\subsubsection{Gradient estimation}
The extremum seeking controller approximates the gradient of the cost function and integrates the negative of the estimated gradient to minimize the cost \cite{multiESC2}.
To approximate the gradient of the cost function, sinusoidal perturbations 
\begin{align}
    \bm{p}(t) = [a_v\sin(\omega_vt), a_\beta \sin(\omega_\beta t)]^T
\end{align}
are added to the speed setpoint $\hat{r}_v$ and sideslip setpoint $\hat{r}_\beta$, where $a_v$ and $a_\beta$ are the speed and sideslip perturbation magnitudes and the $\omega_v$ and $\omega_\beta$ are the speed and sideslip perturbation frequencies.
\par
The cost function's value $y$ consists of low-frequency components ($\eta_v$ and $\eta_\beta$) and high-frequency components ($y - \eta_v$ and $y - \eta_\beta$).
The cost is first high pass filtered to remove the low-frequency components and retain only the cost changes because of the perturbations.
These values are then multiplied elementwise with the demodulation signals  
\begin{align}
\bm{d}(t) = [\sin(\omega_vt), \sin(\omega_\beta t)]^T,
\end{align}
where the demodulation signals' frequencies $w_v$ and $w_\beta$ are the same as their corresponding perturbation frequencies.
We denote the results of the multiplications as $\xi_v$ and $\xi_\beta$.
If the cost function's value change is in phase with the perturbations, which means that the cost value increases as the inputs' values increase, $\xi_v$ and $\xi_\beta$ will be positive. 
If they are out of phase, the outputs will be negative.
After this, $\xi_v$ and $\xi_\beta$ are sent to low pass filters, whose outputs are approximations of the cost function's gradient, denoted by $q_v$ and $q_\beta$.

\subsubsection{Step size adapter}
\label{subsubsec:step_size_adapter}
The difference between the proposed extremum seeking controller and the standard multivariable extremum seeking controller \cite{esc_stability} is the step size adapters, which are defined as follows:
\begin{align}
    & \dot{m}_v = \gamma_v (q_v^2 - m_v), \quad \dot{m}_\beta = \gamma_\beta (q_\beta^2 - m_\beta), \label{equ:second_moment_lp} \\
    & g_v = \frac{q_v}{\sqrt{m_v + \epsilon}}, \quad g_\beta = \frac{q_\beta}{\sqrt{m_\beta + \epsilon}}, \label{equ:adapter_output} 
\end{align} 
where $ m_v, m_\beta$ are estimates of the second moments of the output of the low pass filters $q_v$ and $q_\beta$, and $\epsilon$ is a small positive constant preventing dividing by zero.  
Equations in \eqref{equ:second_moment_lp} are essentially first-order low-pass filters for $q_v^2$ and $q_\beta^2$, and $\gamma_v$ and $\gamma_\beta$ denote their cut-off frequencies respectively. 
The idea is motivated by the adaptive moment estimation algorithm (Adam) \cite{Adam}, which is commonly used in the stochastic optimization of objective functions in machine learning, such as training neural networks \cite{GAN,allYouNeed}. 
\par
The adapters take in the output of the low pass filters $q_v$ and $q_\beta$ (the gradient estimates), and outputs $g_v$ and $g_\beta$. They are then passed to the integrators to perform gradient descent.
The effective step size for gradient descent is ${k_v}{g_v}/{q_v}$ for the speed optimization and ${k_\beta}{g_\beta}/{q_\beta}$ for the sideslip optimization, and the step size adapters change them by changing $g_v$ and $g_\beta$.
The second moments of the initial outputs from the low pass filters are used to initialize $m_v$ and $m_\beta$ in \eqref{equ:second_moment_lp}.
\par
In \eqref{equ:adapter_output}, by dividing $q_v$ and $q_\beta$ with the square root of their corresponding second moments, the outputs $g_v$ and $g_\beta$ of the adapters will be approximately bounded by $\pm 1$, since $|\mathbb{E}[q_l]|/\sqrt{\mathbb{E}[q_l^2]} \le 1$ ($\mathbb{E}$ denotes expected value, and $q_l$ being either $q_v$ or $q_\beta$).
As a result, the descent rates for speed and slideslip are bounded by $k_v$ and $k_\beta$.
This can be understood as establishing a trust region around the current parameter value, beyond which the current gradient estimation can be inaccurate. 
In addition, the adapters output small values when the gradient estimates have large uncertainty ($m_v$ and $m_\beta$ are large) and vice versa, which makes the controller more robust to noise.
\subsection{Stability Analysis}
\label{subsec:stability_proof}
In this section, we present the stability proof of the novel multivariable extremum seeking controller with adaptive step size through averaging and singular perturbation analysis. 
A similar methodology was used in \cite{esc_stability} to prove the stability of a single variable standard extremum seeking controller and was used in \cite{esc_newton} to prove the stability of a multivariable Newton-based extremum seeking controller.

\subsubsection{System dynamics}
By substituting the setpoint $\bm{r}$ with $\hat{\bm{r}}+\bm{p}(t)$, the closed-loop dynamics of the vehicle in \eqref{eq:plant_dynamics} can be rewritten as 
\begin{equation}
    \begin{split}
    & \dot{\bm{x}} = \bm{f} (\bm{x}, \bm{\alpha}(\bm{x}, \hat{\bm{r}}+\bm{p}(t))). \label{eq:system_dynamics_plant2}
    \end{split}
\end{equation}
The proposed extremum seeking controller's dynamics in \figref{fig:escDiagram} can be summarized as
\begin{equation}
\begin{split}
    &\dot{\hat{r}}_v = -k_v \frac{q_v}{\sqrt{m_v + \epsilon}}, \quad \dot{\hat{r}}_\beta = -k_\beta \frac{q_\beta}{\sqrt{m_\beta + \epsilon}},\\
    &\dot{q}_v = -\omega_{lv}q_v + \omega_{lv} (y - \eta_v) \sin{w_v t}, \\
    &\dot{q}_\beta = -\omega_{l\beta}q_\beta + \omega_{l\beta} (y - \eta_\beta) \sin{w_\beta t}, \\
    &\dot{\eta}_v = -\omega_{hv}\eta_{v} + \omega_{hv} y, \quad \dot{\eta}_\beta = -\omega_{h\beta}\eta_\beta + \omega_{h\beta} y, \\
    &\dot{m}_v = \gamma_v (-m_v + q_v^2), \quad \dot{m}_\beta = \gamma_\beta (-m_\beta + q_\beta^2). \label{eq:system_dynamics_esc}
\end{split}
\end{equation}

The parameters for the extremum seeking controller are selected as 
\begin{equation}
\begin{split}
    &\omega_v=\omega\omega_v^{'}=O(\omega), \quad  \omega_\beta=\omega\omega_\beta^{'}=O(\omega),\\
    &\omega_{hv}=\omega\delta w_{hv}^{'}=O(\omega\delta), \quad \omega_{h\beta}=\omega\delta w_{h\beta}^{'}=O(\omega\delta), \\
    &\omega_{lv}=\omega\delta w_{lv}^{'}=O(\omega\delta), \quad \omega_{l\beta}=\omega\delta w_{l\beta}^{'}=O(\omega\delta),\\ 
    &k_v = \omega\delta k_v^{'} = O(\omega\delta), \quad k_\beta = \omega\delta k_\beta^{'} = O(\omega\delta),\\
    &\gamma_v = \omega\delta \gamma_v^{'} = O(\omega\delta), \quad  \gamma_\beta = \omega\delta \gamma_\beta^{'} = O(\omega\delta), \label{equ:prime_symbols}
\end{split}
\end{equation}
where $\delta$ and $\omega$ are small positive constants, and $\omega_v^{'}$, $\omega_\beta^{'}$, $\omega_{hv}^{'}$, $\omega_{h\beta}^{'}$, $\omega_{lv}^{'}$, $\omega_{l\beta}^{'}$, $k_v^{'}$, $k_\beta^{'}$, $\gamma_v^{'}$ and $\gamma_\beta^{'}$ are positive constants.
In addition, for this multivariable extremum seeking controller to work for both the speed and the sideslip angle  simultaneously, their perturbation frequencies $\omega_v$ and $\omega_\beta$ should be distinct.

For the following averaging and singular perturbation analysis, we use the time scale $\tau = \omega t$. In addition, we define
\begin{equation}
    \label{equ:pertubations}
    \begin{split}
        \tilde{r}_v = \hat{r}_v - r_v^*, &\quad \tilde{r}_\beta = \hat{r}_\beta - r_\beta^*, \\
        \tilde{\eta}_v = \eta_v - (h \circ \bm{l})(\bm{r}^*), &\quad \tilde{\eta}_\beta = \eta_\beta - (h \circ \bm{l})(\bm{r}^*).
    \end{split}
\end{equation}
Then, the system dynamics in \eqref{eq:system_dynamics_plant2} and \eqref{eq:system_dynamics_esc} with small perturbations can be rewritten as:
\begin{align}
    \omega \frac{d \bm{x}}{d\tau} = \bm{f} (\bm{x}, \bm{\alpha}(\bm{x}, \bm{r^*} +\tilde{\bm{r}}+\bm{p}(\tau))), \label{eq:rewritten_plant} 
\end{align}

\begin{align}
    \frac{d}{d \tau} \mmat{\tilde{r}_v \\ \tilde{r}_\beta \\ q_v \\ q_\beta \\ \tilde{\eta}_v \\ \tilde{\eta}_\beta \\ m_v \\ m_\beta} = 
    \delta \mmat{ (-k_v^{'} q_v)/\sqrt{m_v + \epsilon} \\[3pt]
                  (-k_\beta^{'} q_\beta)/\sqrt{m_\beta + \epsilon} \\[3pt]
                  \omega_{lv}^{'} (y - (h \circ \bm{l})(\bm{r}^*) - \tilde{\eta}_v) \sin{w_v^{'} \tau} - \omega_{lv}^{'}q_v \\[3pt]
                  \omega_{l\beta}^{'} (y - (h \circ \bm{l})(\bm{r}^*) - \tilde{\eta}_\beta) \sin{w_\beta^{'} \tau} - \omega_{l\beta}^{'}q_\beta \\[3pt]
                  -\omega_{hv}^{'}\tilde{\eta}_v + \omega_{hv}^{'} (y - (h \circ \bm{l})(\bm{r}^*)) \\[3pt]
                  -\omega_{h\beta}^{'}\tilde{\eta}_v + \omega_{h\beta}^{'} (y - (h \circ \bm{l})(\bm{r}^*)) \\[3pt]
                  \gamma_v^{'} (-m_v + q_v^2) \\[3pt]
                  \gamma_\beta^{'} (-m_\beta + q_\beta^2)
                } \label{equ:rewritten_esc}
\end{align}
where $\tilde{\bm{r}} = [\tilde{r}_v, \tilde{r}_\beta]^T$, $\bm{\Bar{p}}(\tau) = \bm{p}(t/\omega)$.

\begin{figure*}[ht]
	\begin{center}    \includegraphics[width =  \linewidth]{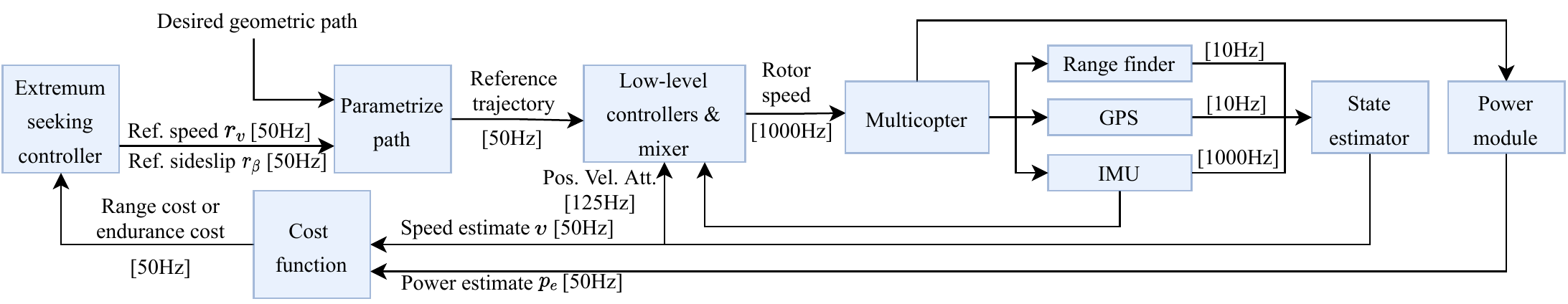}
	\end{center}
	\vspace{-1ex}
	\caption{Control architecture for the model-free adaptive flight range or endurance optimization of a multicopter. The details of the extremum seeking controller block is shown in the dashed rectangle in \figref{fig:escDiagram}. The extremum seeking controller runs on the onboard computer (Jetson Nano) while the low-level controllers and the state estimator run on the Pixracer flight controller.}
	\label{fig:escFullDiagram}
\end{figure*}

\subsubsection{Averaging analysis}
\label{sec:averaging}
We first freeze the dynamics of the vehicle \eqref{eq:system_dynamics_plant2} at its equilibrium point $\bm{x} = \bm{l}(\bm{r}^* + \tilde{\bm{r}} + \bm{\Bar{p}}(\tau))$, substitute it into \eqref{equ:rewritten_esc} and get the reduced system
\begin{align}
    \frac{d}{d \tau} \mmat{\tilde{r}_v \\ \tilde{r}_\beta \\ q_v \\ q_\beta \\ \tilde{\eta}_v \\ \tilde{\eta}_\beta \\ m_v \\ m_\beta} = 
    \delta \mmat{ (-k_v^{'} q_v)/\sqrt{m_v + \epsilon} \\[3pt]
                  (-k_\beta^{'} q_\beta)/\sqrt{m_\beta + \epsilon} \\[3pt]
                  \omega_{lv}^{'} (v(\bm{\tilde{r}} + \bm{\Bar{p}}(
                  {\tau})) - \tilde{\eta}_v) \sin{w_v^{'} \tau} - \omega_{lv}^{'}q_v \\[3pt]
                  \omega_{l\beta}^{'} (v(\bm{\tilde{r}} + \bm{\Bar{p}}(\tau)) - \tilde{\eta}_\beta) \sin{w_\beta^{'} \tau} - \omega_{l\beta}^{'}q_\beta \\[3pt]
                  -\omega_{hv}^{'}\tilde{\eta}_v + \omega_{hv}^{'} v(\bm{\tilde{r}} + \bm{\Bar{p}}(\tau)) \\[3pt]
                  -\omega_{h\beta}^{'}\tilde{\eta}_v + \omega_{h\beta}^{'} v(\bm{\tilde{r}} + \bm{\Bar{p}}(\tau)) \\[3pt]
                  \gamma_v^{'} (-m_v + q_v^2) \\[3pt]
                  \gamma_\beta^{'} (-m_\beta + q_\beta^2)
                }, \label{equ:reduced_model}
\end{align}
\\
where $v(\bm{\tilde{r}} +\bm{\Bar{p}}(\tau)) = (h \circ \bm{l})(\bm{r}^* + \tilde{\bm{r}} + \bm{\Bar{p}}(\tau)) - (h \circ \bm{l})(\bm{r}^*)$. From Assumption 3 we have that: 
\begin{equation}
    v(0) = 0, \ \triangledown v(0) = 0, \ \triangledown^2 v(0) > 0. \label{minimum-perturbation-dynamics}
\end{equation}
To provide compact notations, we denote $\triangledown^2 v(0) = H$ for later discussion. 
The least common period of sinusoidal functions with frequencies of $\omega_v^{'}$ and $ \omega_\beta^{'}$ is defined as $\Pi$.
We first prove the stability of the reduced system using averaging analysis:
\begin{prop}
\label{prop:neighborhood-convergence}
For the reduced system \eqref{equ:reduced_model}, under Assumption 3, there exists $ \Bar{a}$ and $\Bar{\delta}$ such that for all $ \mnorm{\bm{a}} \in (0, \Bar{a})$, $ \delta \in (0, \Bar{\delta})$, the reduced system dynamics \eqref{equ:reduced_model} have a unique exponentially stable periodic solution of period $\Pi$, which for all $ \tau > 0$
\begin{align}
\begin{split}
    &\left|\tilde{r}^{\Pi}_v(\tau) \right| \le O(\delta + \mnorm{\bm{a}}^2), \; \; \left|\tilde{r}^{\Pi}_\beta(\tau) \right| \le O(\delta + \mnorm{\bm{a}}^2), \\
    &\left|\tilde{\eta}^{\Pi}_v(\tau) \right| \le O(\delta + \mnorm{\bm{a}}^2), \; \; \left|\tilde{\eta}^{\Pi}_\beta(\tau)\right| \le O(\delta + \mnorm{\bm{a}}^3), \\
    &\left| q^{\Pi}_v \right| \le O(\delta), \; \; \left| q^{\Pi}_\beta \right| \le O(\delta), \\
    &\left| m^{\Pi}_v \right| \le O(\delta), \; \; \left| m^{\Pi}_\beta \right| \le O(\delta).
\end{split}
\end{align}
\end{prop}
\begin{proof}
The proof of Proposition \ref{prop:neighborhood-convergence} is shown in the appendix at the end of this paper.
\end{proof}
This implies that the error terms $\tilde{r}^{\Pi}_v(\tau)$ and $\tilde{r}^{\Pi}_\beta(\tau)$ converge to an $O(\delta + \mnorm{\bm{a}}^2)$ neighbourhood of zero.
The flight speed and sideslip found by the extremum seeking controller are periodic and converge to an $O(\delta + \mnorm{\bm{a}}^2)$ neighbourhood of their optimal values $r_v^*$ and $r_\beta^*$ (i.e. values that minimize the cost functions defined in \secref{sec:problem_statement}). 
\subsubsection{Singular perturbation analysis}
We then analyze the full system \eqref{eq:rewritten_plant} and \eqref{equ:rewritten_esc}.
To provide compact notations, we define the state vector of the extremum seeking controller as $\bm{z} = [\tilde{r}_v, \tilde{r}_\beta, q_v, q_\beta, \tilde{\eta}_v, \tilde{\eta}_\beta, m_v, m_\beta]^T$, and write \eqref{equ:rewritten_esc} as 
\begin{equation}
    \frac{d\bm{z}}{d\tau} = \delta \bm{E}(\tau, \bm{x}, \bm{z}). \label{eq:rewritten_esc2}
\end{equation}
By Proposition 1, there exists an exponentially stable periodic solution $\bm{z}^{\Pi}(\tau)$ such that
\begin{equation}
    \frac{d \bm{z}^{\Pi}(\tau)}{d\tau} = \delta \bm{E}(\tau, \bm{L}(\tau, \bm{z}^{\Pi}(\tau)), \bm{z}^{\Pi}(\tau)).
\end{equation}
where $\bm{L}(\tau, \bm{z}^{\Pi}(\tau)) = \bm{l}(\bm{r}^* + \tilde{\bm{r}} + \bm{\Bar{p}}(\tau))$.
To convert the system \eqref{eq:rewritten_plant} and \eqref{eq:rewritten_esc2} into the standard singular perturbation form, we shift the state $\bm{z}$ to get $\tilde{\bm{z}} = \bm{z} - \bm{z}^{\Pi}(\tau)$ such that 
\begin{align}
    &\omega \frac{d\bm{x}}{d\tau} = \tilde{\bm{F}}(\tau, \bm{x}, \tilde{\bm{z}}),\\
    &\frac{d \tilde{\bm{z}}}{d\tau} = \delta \tilde{\bm{E}}(\tau, \bm{x}, \tilde{\bm{z}}) \label{eq:shifted_esc_model}. 
\end{align}
where
\begin{align}
    &\tilde{\bm{E}}(\tau, \bm{x}, \tilde{\bm{z}}) := \bm{E}(\tau, \bm{x}, \tilde{\bm{z}}+\bm{z}^{\Pi}(\tau)) - \bm{E}(\tau, \bm{L}(\tau, \bm{z}^{\Pi}(\tau)), \bm{z}^{\Pi}(\tau)) \nonumber \\
    &\tilde{\bm{F}}(\tau, \bm{x}, \tilde{\bm{z}}) := \bm{f} (\bm{x}, \bm{\alpha}(\bm{x}, \bm{r^*} +\tilde{\bm{r}}+\bm{p}(\tau))). \nonumber
\end{align}
The quasi-steady state is 
\begin{equation}
    \bm{x} = \bm{L}(\tau, \tilde{\bm{z}} + \bm{z}^{\Pi}(\tau)).
\end{equation} 
By substituting the quasi-steady state into \eqref{eq:shifted_esc_model} and we get the reduced model
\begin{equation}
    \frac{d \tilde{\bm{z}}}{d \tau} = \delta \tilde{\bm{E}}(\tau,  \bm{L}(\tau, \tilde{\bm{z}} + \bm{z}^{\Pi}(\tau)), \tilde{\bm{z}}), \label{eq:reduced_model}
\end{equation}
which has an equilibrium at the origin $\tilde{\bm{z}} = 0$. The equilibrium has been shown to be exponentially stable in the proof of Proposition 1. 
In addition, we study the stability of the boundary layer model (in the time scale $t = \tau / \omega$)
\begin{align}
    \frac{d \bm{x}_b}{d t} &= \tilde{\bm{F}}(\tau, \bm{x}_b + \bm{L}(\tau, \tilde{\bm{z}} + \bm{z}^{\Pi}(\tau)), \tilde{\bm{z}}) \\
    &= \bm{f}(\bm{x}_b + \bm{l}(\bm{r}), \bm{\alpha}(\bm{x}_b + \bm{l}(\bm{r}), \bm{r})). \label{eq:boundary_layer_model}
\end{align}
Since $\bm{f}(\bm{l}(\bm{r}), \bm{\alpha}(\bm{l}(\bm{r}), \bm{r})) = 0$ according to Assumption 1, $\bm{x}_b = 0$ is the equilibrium of the boundary layer model \eqref{eq:boundary_layer_model}. 
By Assumption 2, this equilibrium is locally exponentially stable uniformly in $\bm{r}$.
\par
Combining the exponential stability of the reduced model with the exponential stability of the boundary layer model, and using Tikhonov's theorem on the infinite interval \cite[Chapter 11.3]{khalil2002nonlinear}, we can conclude that the solution of \eqref{eq:rewritten_esc2} is $O(\omega)$-close to the solution of the reduced model \eqref{eq:reduced_model}.
Using the results of Proposition 1, we can then conclude that the error terms $\tilde{r}^{\Pi}_v(\tau)$ and $\tilde{r}^{\Pi}_\beta(\tau)$ converge to an $O(\omega + \delta + \mnorm{\bm{a}}^2)$ neighbourhood of zero.
\par
In summary, the proposed extremum seeking controller is locally stable -- starting from an initial condition near the cost function's local minimum, it will converge to a neighbourhood around that local minimum if the perturbation is sufficiently small and slow relative to the closed-loop dynamics of the vehicle, and if the Assumptions 1-3 hold.

\section{Experimental results}
\label{sec:results}
\begin{figure*}[!htp]
    \centering
    \subfigure[Range cost with box payload.]
    {
        \includegraphics[width =  0.23\linewidth]{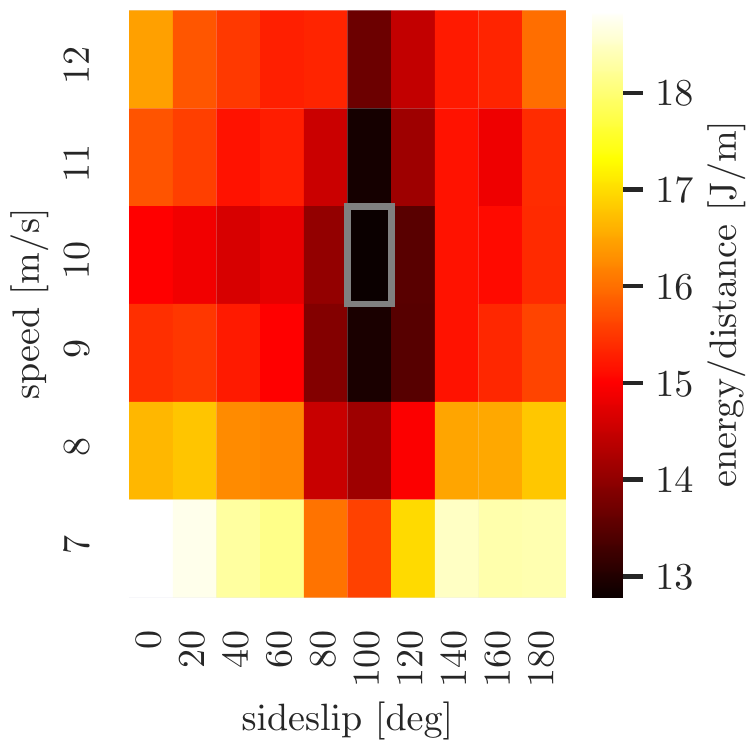}
        \label{subfig:groundTruthRangeWithBox}
    }
    \subfigure[Endurance cost with box payload.]
    {
        \includegraphics[width =  0.23\linewidth]{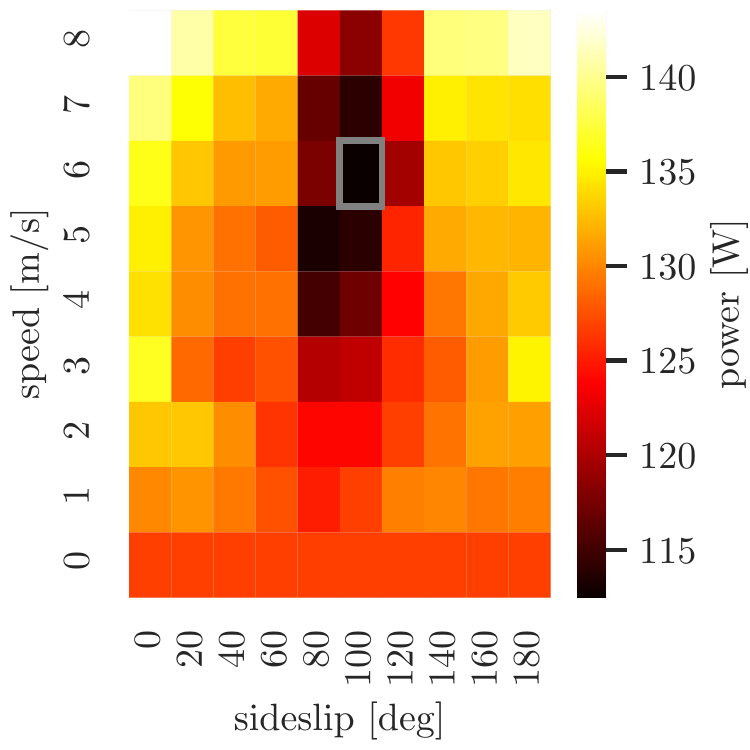}
        \label{subfig:groundTruthEnduranceWithBox}
    }
    \subfigure[Range cost with no box payload.]
    {
        \includegraphics[width =  0.23\linewidth]{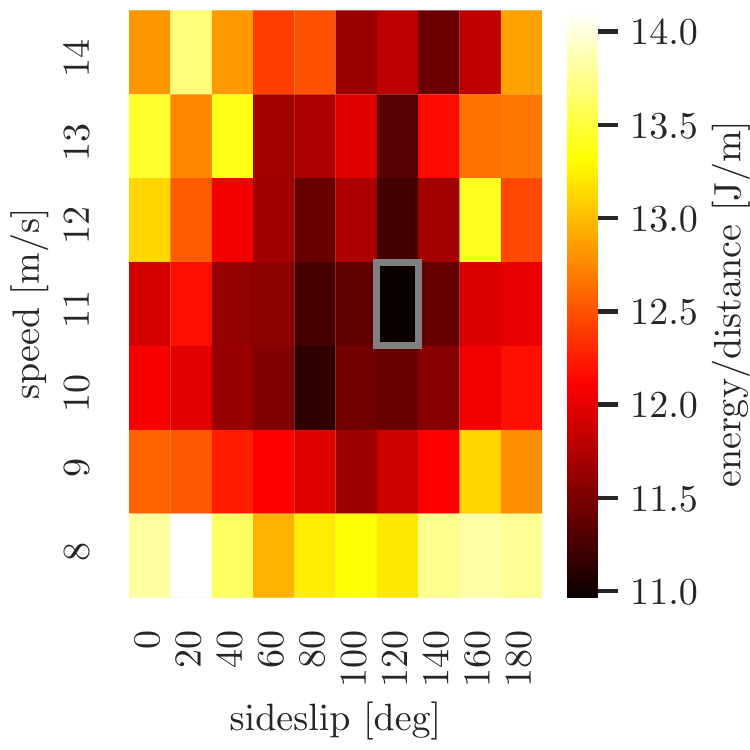}
        \label{subfig:groundTruthRangeNoBox}
    }
    \subfigure[Endurance cost with no box payload.]
    {
        \includegraphics[width =  0.23\linewidth]{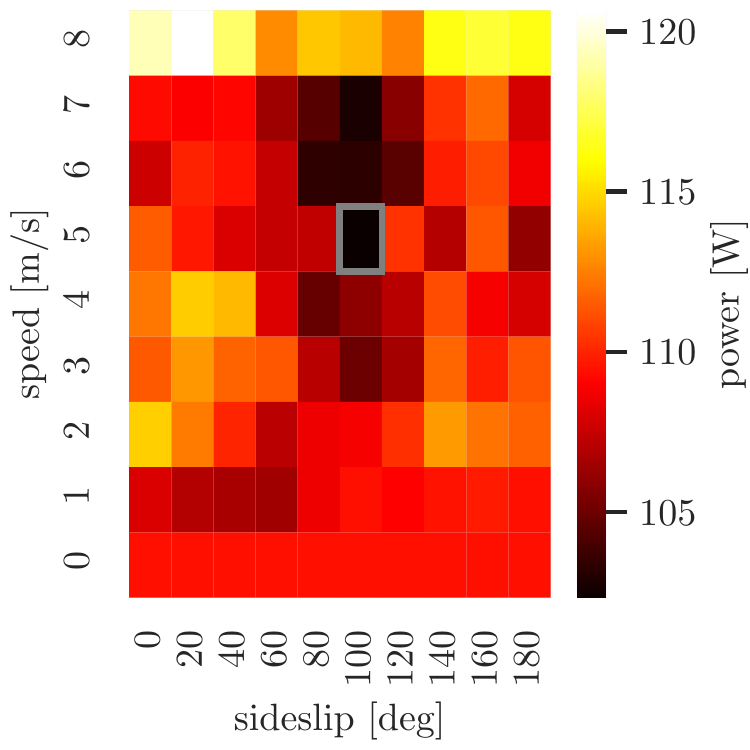}
        \label{subfig:groundTruthEnduranceNoBox}
    }
    \vspace{-1ex}
    \caption{Ground truth data of the cost functions' values, with and without an additional box payload. 
    Each square in the heat maps corresponds to 20 seconds' data collected at 50Hz.
    The optimal value in each case is encircled with a grey rectangle.
    (a) The range cost with the box payload reaches its minimum at about 10 m/s in speed and 100 degrees in sideslip.
    (b) The endurance cost with the box payload reaches its minimum at about 6 m/s and 100 degrees in sideslip.
    (c) The range cost with no box payload reaches its minimum at about 11 m/s and 120 degrees in sideslip.
    (d) The endurance cost with no box payload reaches its minimum at about 5 m/s and 100 degrees in sideslip.
    }
    \label{fig:groundTruth}
\end{figure*}

Outdoor experiments were conducted to demonstrate  the effectiveness of the extremum seeking controller with adaptive step size to find the optimal flight speed and sideslip.
The proposed method was shown to have better convergence speed than the standard extremum seeking control.
It was also able to converge in the present of strong wind disturbances.
An experiment video can be found at \url{https://youtu.be/aLds8LVfogk}.

\par
We want to note the significance of the outdoor experiments in this paper compared to indoor experiments in our previous work of \cite{previous_paper_1, previous_paper_2}:
\begin{enumerate}
    \item In our previous work, a motion caption system was used to measure the vehicle's position and attitude at very high accuracy (about 1 mm error for position and 1 degree error for attitude) and at 200 Hz frequency.
    In contrast, in the outdoor experiments, a GPS was used for position estimation, whose accuracy was at meter level with a much lower frequency (10 Hz).
    As motion capture systems are not available in most of the real-world applications, this new sensor setup with GPS shows that our proposed method is able to perform well under much larger state estimation variances compared with indoor experiments.
    \item Because of limited space, the multicopter was only able to fly a circular path of 2 m radius in previous indoor experiments. 
    The centripetal force increased dramatically as the flight speed increased for such a small radius, contributing largely to the power consumption.
    Such an experimental setup is rare in real-world applications such as package delivery or surveillance, and made the vehicle's power consumption increase almost monotonically as the speed increased.
    In outdoor experiments, the centripetal force became much smaller due to much larger flight radius -- a more realistic experiment setup.
    We were thus also able to find the speed and sideslip for optimal endurance flight, as the power as a function of speed and sideslip has a much deeper minimum.
    \item Experiments were conducted both on light wind and windy days, to see the effect of wind disturbances on the proposed method. Such real-world effects were not possible indoors.
    \item We used the standard, off-the-shelf PX4 firmware for the low-level control and state estimation of the vehicle, instead of using a custom firmware and control stack in our previous work. 
    This demonstrates the ability of the proposed method to be easily deployed on existing multicopters.
\end{enumerate}
\subsection{Experiment setup}
The experiments were performed with a custom-built quadcopter with and without a box payload (as shown in \figref{fig:front}).
The weight of the vehicle without the box payload was 0.9 kg, and the box weighs 0.1 kg and has a size of 180$\times$115$\times$80 mm. 
The distance between the hubs of the two  diagonal motors is 330 mm and the propeller is 203 mm in diameter. 
The extremum seeking controller was run on an onboard computer (Jetson Nano), and an  mRo Pixracer R15 flight controller ran the standard PX4 firmware \cite{px4_firmware} including the state estimator and low-level controllers. 
The low-level cascaded PID controller (corresponds to $\bm{\alpha}(\bm{x}, \bm{r})$ in Assumption 2) stabilizes the vehicle and thus satisfies Assumption 2 in \secref{sec:esc}.
Other low-level controllers satisfying Assumption 2 could also be used with our proposed method.
The Jetson Nano and the Pixracer communicate through a UART link using mavros. 
The main reasons for running the extremum seeking controller on the onboard computer are for easier data logging and implementation.
The computational power of micro controllers such as the Pixracer should also be able to run this algorithm, as it only requires several simple operations as shown in \figref{fig:escDiagram}. 
Removing the onboard computer could further save the energy, at the cost of not being to log data as easily.
The experiments were conducted at a flat grass field at the Richmond Field Station, Richmond, CA (37.916588 N, -122.336667 E).

\par
The control architecture for the vehicle is shown in \figref{fig:escFullDiagram}.
The extremum seeking controller (with or without adaptive step size) takes in the desired geometric path and instantaneous range cost or endurance cost.
The power measurement $p_e$ comes from a power module (Holybro PM06 v2) connected to the battery, and the speed measurement $v$ comes from a state estimator based on a GPS (Zubax GNSS 2), a range finder for measuring the flight height (Beneware TFmini-S) and an IMU (Invensense MPU-9250). 
The extremum seeking controller outputs the reference tangential speed $r_v$ and sideslip $r_\beta$ along the desired path, which are used to parameterize the geometric path into a reference trajectory.
The reference trajectory is then tracked by the low-level position and attitude controller, which is a cascaded PID controller.
\par
The range of flight speed was 0-12 m/s when carrying the box payload and was 0-15 m/s without payload. 
The sideslip angle is a periodic variable, whose period is $180^\circ$, due to the vehicle and payload's rotational symmetry.


\subsection{Extremum seeking parameter selection}
\label{sec:esc_parameter_selection}
The values of parameters of the standard extremum seeking controller and our proposed adaptive step size extremum seeking controller used throughout the experiments are shown in Table \ref{tab:paramter}.
The perturbation frequencies ($w_v$ and  $w_\beta$), perturbation magnitudes ($a_v$ and $a_\beta$), gains for the integrator ($k_v$ and $k_\beta$), cutoff frequencies of high-pass ($w_{hv}$ and $w_{h\beta}$) and low-pass filters ($w_{lv}$ and $w_{l\beta}$) need to be selected properly to achieve good performance of the extremum seeking controllers. 
The guidelines for choosing them are detailed below:

\subsubsection{Perturbation frequencies} 
The perturbation frequencies must be slow compared with the closed-loop dynamics of the quadcopter ($\omega$ should be small as mentioned in the stability analysis), such that they can be well  tracked by the vehicle.
Mathematically, the perturbation frequency could be selected smaller than the dominant frequency of the vehicle's closed-loop dynamics.
The perturbation frequencies can be increased to achieve a faster convergence rate \cite{ditheringSignalStudy}, given they can be tracked well by the vehicle. 
In addition, the multivariable extremum seeking control requires distinct perturbation frequencies for the speed and sideslip angle.
\subsubsection{Perturbation magnitudes and integrator gains} 
Large values for the perturbation magnitudes will be helpful for faster convergence, but will increase the oscillation magnitudes. 
Large values for the integrator gains will also be helpful for faster convergence, but will make the controller more sensitive to disturbances.
As a result, we can increase the perturbation magnitudes and integrator gains to obtain the fastest convergence speed for a permissible amount of oscillation and sensitivity. 
\subsubsection{Cutoff frequencies of the high-pass and low-pass filters} The cutoff frequencies of the high-pass and low-pass filters should be designed based on their corresponding perturbation frequencies: the cutoff frequency of the high-pass filter should be set higher than the perturbation frequency ($w_{hv}{\ge}w_v$ and $w_{h\beta}{\ge}w_v$), and the cutoff frequency of the low-pass filter should be set lower than the perturbation frequency ($w_{lv}{\le}{w_\beta}$ and $w_{l\beta}{\le}{w_\beta}$), to prevent attenuation of measurements at the perturbation frequency. 
We set the cutoff frequencies of the high-pass and low-pass filters to be the same as their corresponding perturbation frequencies, which simplified the parameter tuning process and was found to work well in the experiments.
\subsubsection{Step-size adapter cutoff frequency} The two parameters in the step size adapters $\gamma_v$ and $\gamma_\beta$ are cutoff frequencies for the low-pass filters of the square for estimated gradient $q_v^2$ and $q_\beta^2$. 
One could increase their values as long as  the noises are sufficiently attenuated.
\par
In general, the selection of the extremum seeking parameters is a tuning process, but the guidelines above are valuable for making parameter tuning effectively. 
\par
To make a fair comparison between the standard and the proposed extremum seeking controller, we kept all parameters for the two different methods to be the same except $k_v$ and $k_\beta$, since they have different meanings for the two methods: the $k_v$ and $k_\beta$ values are the step sizes for the standard method but are only part of the step sizes for the adaptive method, as shown in \secref{subsubsec:step_size_adapter}. 
They were empirically tuned in experiments for the two different methods to each achieve the fastest convergence rate in optimal range speed and sideslip searching when carrying a box payload \ref{subfig:escRangeWithBox}.

\begin{table}[ht]
	\renewcommand{\arraystretch}{1.2}
	\caption{Values of extremum seeking parameters}
	\vspace{-1ex}
	\label{tab:paramter}
	\centering
    \begin{tabular}{|c|cc|}
    \hline
    Parameter & \multicolumn{1}{c|}{Standard method} & Proposed method \\ \hline
    $a_v$ & \multicolumn{2}{c|}{0.5 m/s} \\ \hline
    $\omega_v, \omega_{hv}, \omega_{lv}$ & \multicolumn{2}{c|}{1 rad/s} \\ \hline
    $a_\beta$ & \multicolumn{2}{c|}{$10^\circ$} \\ \hline
    $\omega_\beta, \omega_{h\beta}, \omega_{l\beta}$ & \multicolumn{2}{c|}{0.5 rad/s} \\ \hline
    $k_v$ & \multicolumn{1}{c|}{0.05} & 0.11 \\ \hline
    $k_\beta$ & \multicolumn{1}{c|}{0.04} & 0.04 \\ \hline
    $\gamma_v, \gamma_\beta$ & \multicolumn{1}{c|}{N/A} & 0.5 rad/s \\ \hline
    \end{tabular}
\end{table}

\subsection{Performance comparison under light wind}
\label{sec:compare_windstill}
In the comparison experiments, the quadcopter was commanded to fly along a circular path with 30 meters in radius and a constant height of 5 meters.
The circular path was chosen for a simple and intuitive comparison, as well as easy experimental implementation, while our proposed method is also applicable to sufficiently smooth geometric paths with more complicated shapes.
The experiments were conducted during good weather to minimize the effect of wind disturbances.
\begin{figure*}[!htp]
    \centering
    \subfigure[Carrying an additional box payload.]
    {
        \includegraphics[width =  0.48\linewidth]{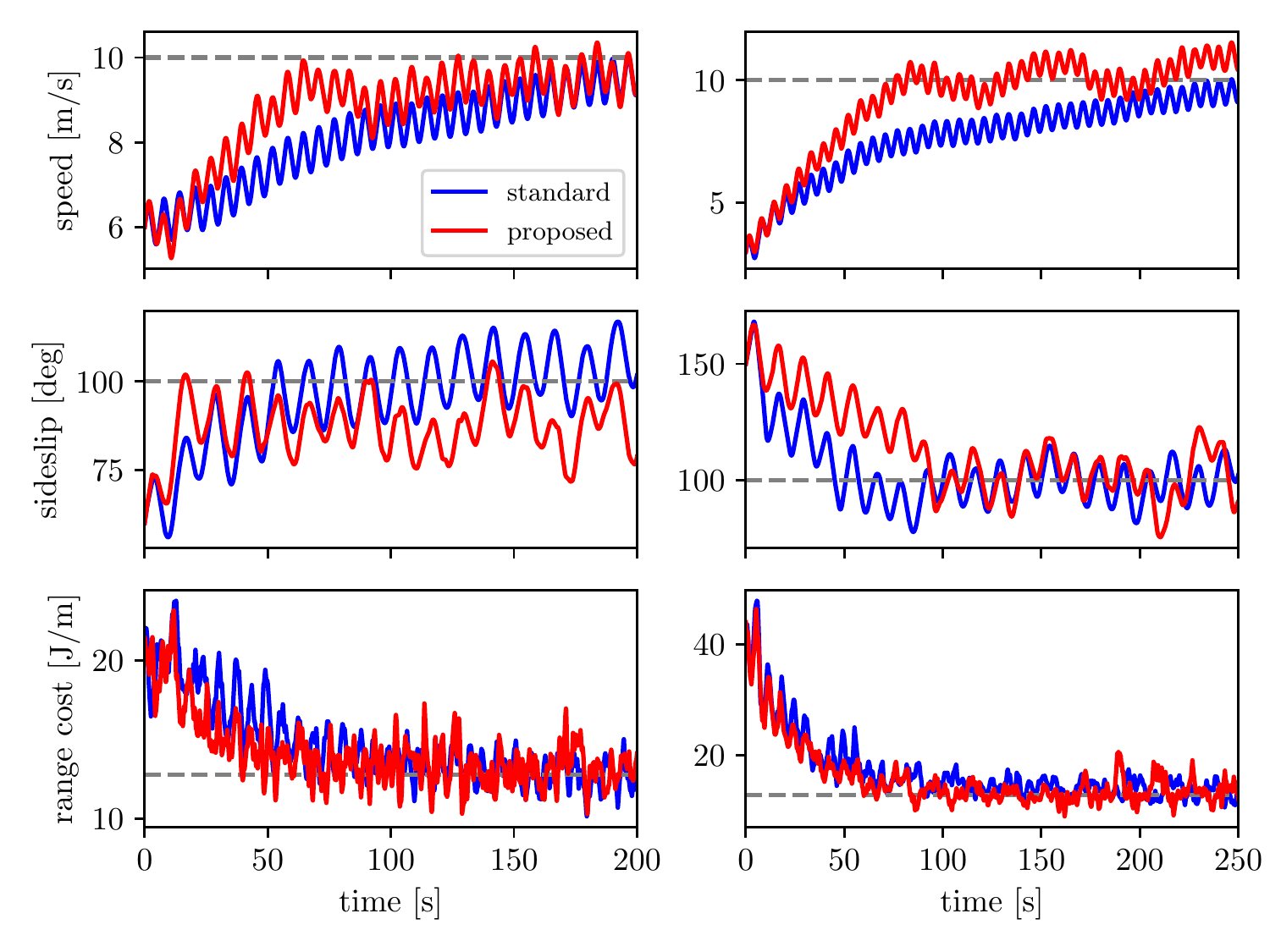}
        \label{subfig:escRangeWithBox}
    }
    \subfigure[Without carrying an additional box payload.]
    {
        \includegraphics[width =  0.48\linewidth]{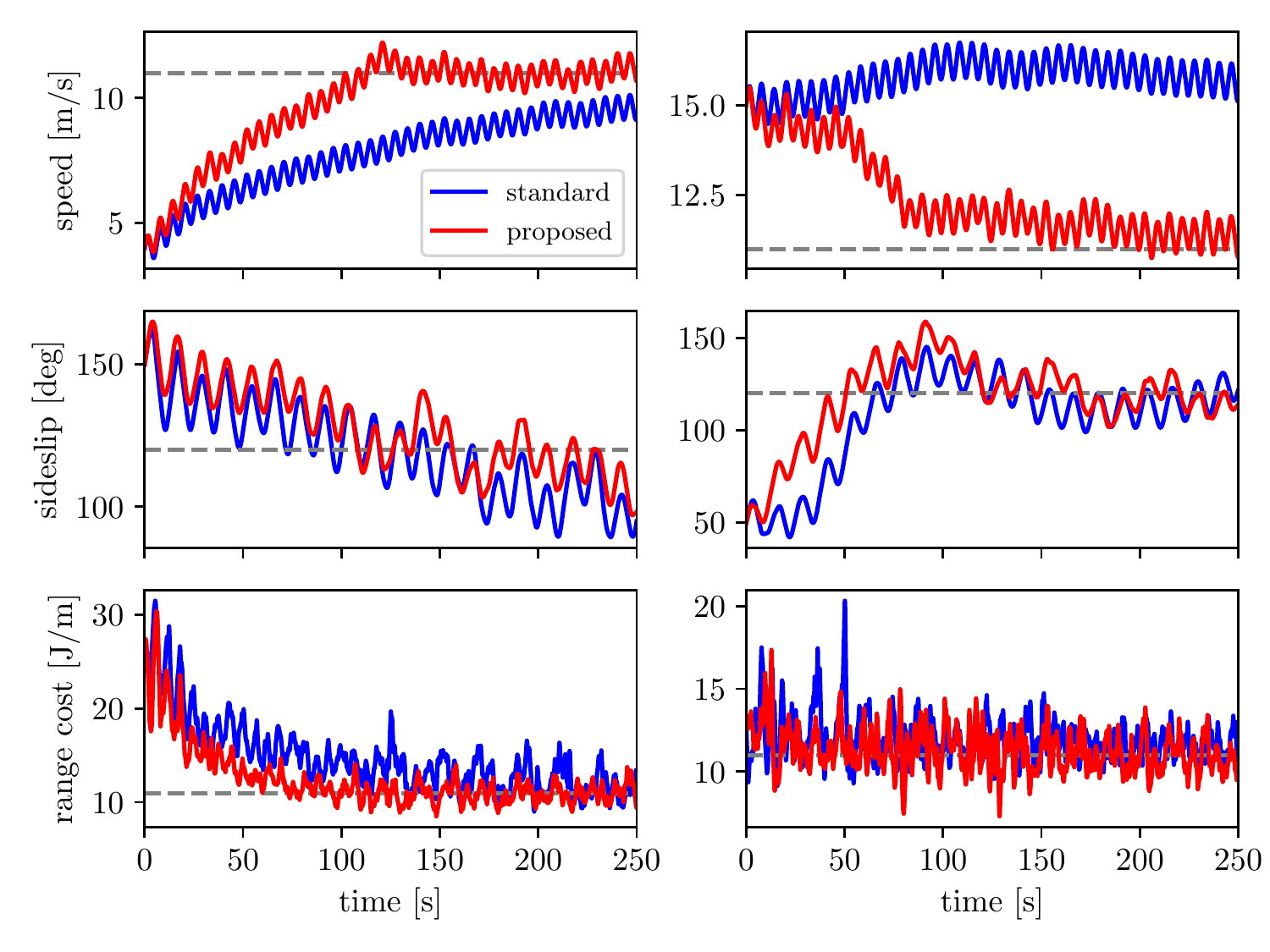}
        \label{subfig:escRangeWithoutBox}
    }
    \vspace{-1ex}
    \caption{Optimal range speed and sideslip searching performance comparison between the proposed method (red lines) and the standard method (blue lines). The ground truth values for optimal speed and sideslip are marked as grey dashed lines (values from \figref{fig:groundTruth}). The results when carrying an additional box payload are shown in (a) and the results with no additional box payload are shown in (b). Each column in the subfigures represent a test with a different initial speed and sideslip.}
    \label{fig:escRange}
\end{figure*}

\begin{figure*}[!htp]
    \centering
    \subfigure[Carrying an additional box payload.]
    {
        \includegraphics[width = 0.48 \linewidth]{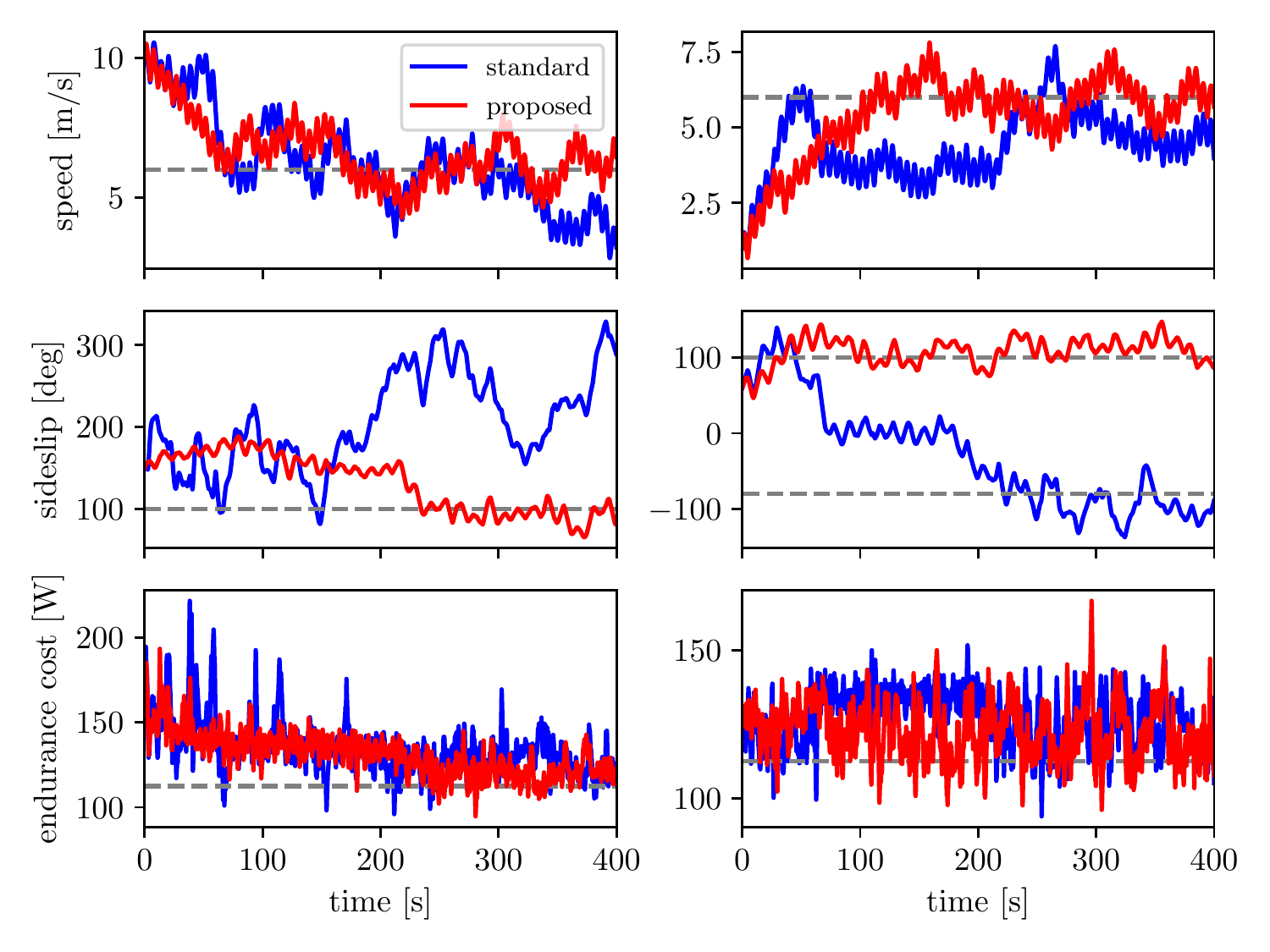}
        \label{subfig:escEnduranceWithBox}
    }
    \subfigure[Without carrying an additional box payload.]
    {
        \includegraphics[width =  0.48\linewidth]{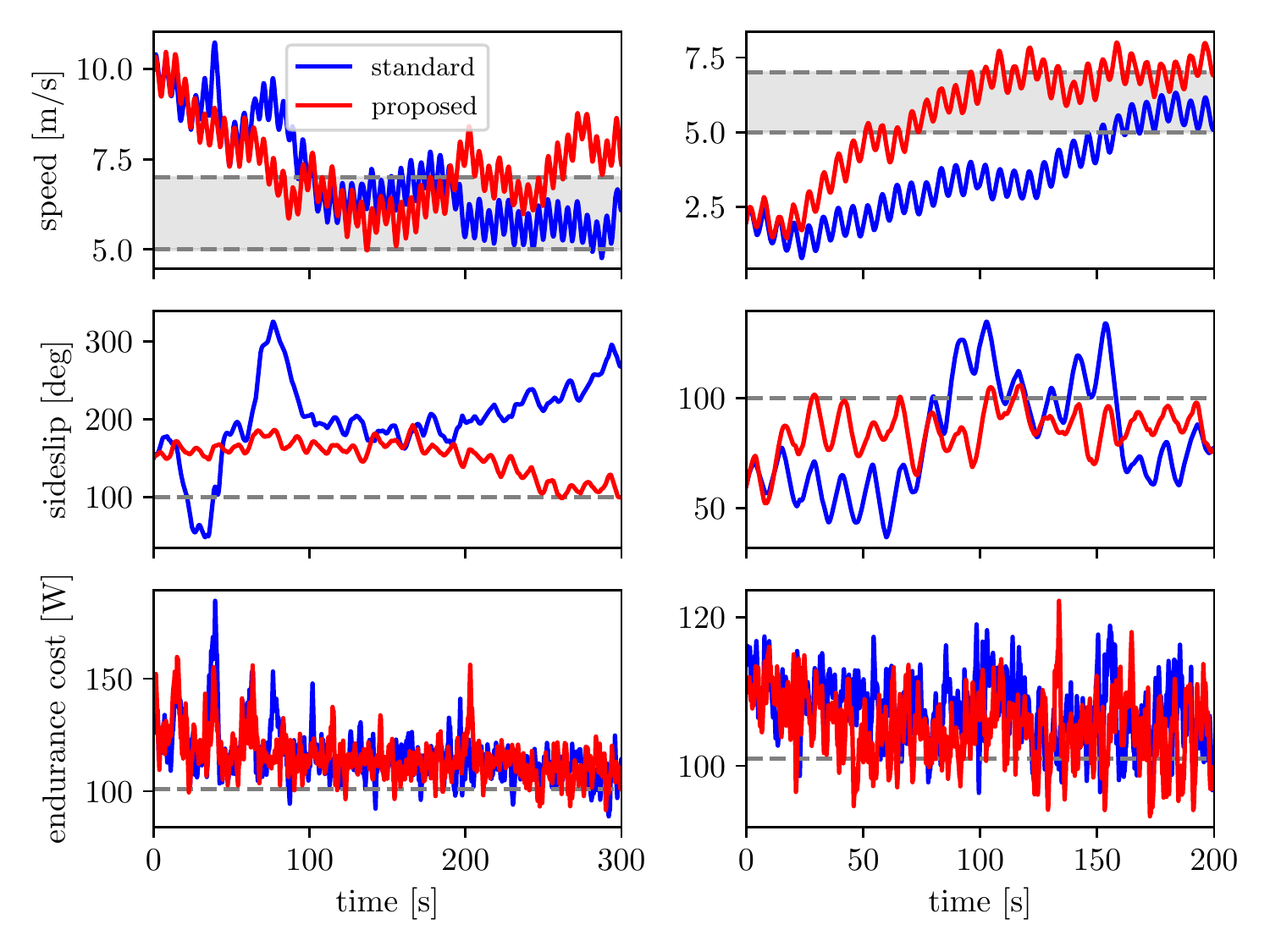}
        \label{subfig:escEnduranceWithoutBox}
    }
    \vspace{-1ex}
    \caption{Optimal endurance speed and sideslip searching performance comparison between the proposed method (red lines) and the standard method (blue lines). The ground truth values for optimal speed and sideslip are marked as grey dashed lines and grey shaded regions (values from \figref{fig:groundTruth}). The results when carrying an additional box payload are shown in (a) and the results with no additional box payload are shown in (b). Each column in the subfigures represent a test with a different initial speed and sideslip. In the second test of (a), the optimal sideslip is marked at both 100 degrees and -80 degrees. This is because the vehicle and payload are rotational symmetric, such that a sideslip offset of 180 degrees has the same effect on the vehicle's power consumption. In (b), the optimal speed is marked as a range between 5 - 7 m/s, because the cost function values are very close in this range with less than 1\% difference. }
    \label{fig:escEndurance}
\end{figure*}

\subsubsection{Cost value ground truth}
\label{sec:cost_value_ground_truth}
To verify that the proposed extremum seeking controller is able to converge close to the optimal speed and sideslip, we experimentally evaluated the optimal range and optimal endurance cost functions.
When the vehicle is carrying the box payload, the values of the cost functions at various speed and sideslip are shown at \figref{subfig:groundTruthRangeWithBox} and \figref{subfig:groundTruthEnduranceWithBox}, while \figref{subfig:groundTruthRangeNoBox} and \figref{subfig:groundTruthEnduranceNoBox} show the values without box.
\par
The data shows the importance of flying at the energy efficient speed and sideslip: compared to flying at the maximum achievable speed in the experiments with a uniformly selected random sideslip, flying at the optimal range speed and sideslip on average increases the flight range by 14.3\% without payload and 19.4\% with a box payload.
Besides, compared to hovering, flying at the optimal endurance speed and sideslip increases the flight time by 7.5\% without payload and 14.4\% with a box payload.

\subsubsection{Convergence speed comparison and discussion}
\label{sec:convergence_comparison}
The convergence speed of the standard and the proposed methods are compared in experiments with/without a box payload and with different initial conditions.
We consider the extremum seeking controller converges when both the speed and sideslip settle close to their optimal value.
When the goal is to find the speed and sideslip which achieve the optimal flight range, the results are compared in \figref{subfig:escRangeWithBox} and \figref{subfig:escRangeWithoutBox}.
When the goal is to find the speed and sideslip which achieve the optimal flight endurance, the results are compared in \figref{subfig:escEnduranceWithBox} and \figref{subfig:escEnduranceWithoutBox}. 
The convergence times are summarized in Table \ref{tab:optimal_range} for optimal range and in Table \ref{tab:optimal_endurance} for optimal endurance (N/A represents that the method failed to converge by the end of the experiment).
We can see that the proposed method converged about twice as fast as the standard method in these tests.

\begin{table}[ht]
\renewcommand{\arraystretch}{1.2}
\caption{Optimal range speed and sideslip seeking}
\vspace{-1ex}
\label{tab:optimal_range}
\centering
\begin{tabular}{|l|l|l|l|l|}
\hline
Payload & Initial speed & Initial sideslip & Standard & Proposed \\ \hline
\multirow{2}{*}{box} & 6 m/s & 60 deg & 200 s & 50 s \\ \cline{2-5} 
 & 3 m/s & 150 deg & 250 s & 100 s \\ \hline
\multirow{2}{*}{none} & 4 m/s & 150 deg & 250 s & 125 s \\ \cline{2-5} 
 & 15 m/s & 50 deg & N/A & 100 s \\ \hline
\end{tabular}
\end{table}
 
\begin{table}[ht]
\renewcommand{\arraystretch}{1.2}
\caption{Optimal endurance speed and sideslip seeking}
\vspace{-1ex}
\label{tab:optimal_endurance}
\centering
\begin{tabular}{|l|l|l|l|l|}
\hline
Payload & Initial speed & Initial sideslip & Standard & Proposed \\ \hline
\multirow{2}{*}{box} & 10 m/s & 150 deg & N/A & 250 s \\ \cline{2-5} 
 & 1 m/s & 60 deg & 200 s & 100 s \\ \hline
\multirow{2}{*}{none} & 10 m/s & 150 deg & N/A & 250 s \\ \cline{2-5} 
 & 2 m/s & 60 deg & 175 s & 75 s \\ \hline
\end{tabular}
\end{table}

In summary, we can see that the proposed extremum seeking controller with step-size adapter converged  about twice as fast as the standard extremum seeking controller.
In addition, the parameters of the extremum seeking controller were tuned for optimal range speed and sideslip searching when carrying a box payload, as mentioned in \secref{sec:esc_parameter_selection}. 
The same set of parameters still worked well for the other experiment setups (optimal endurance goal, with and without box payload) for the proposed method, showing that the method has good robustness to parameters. 
However, the standard extremum seeking method failed to converge in some cases, suggesting it is less robust.
\par
Like other perturbation-based extremum seeking methods, the convergence speed of the proposed method is still limited by the time-scale separation, which requires the changing of the speed and sideslip setpoints to be slow compared to the perturbation frequencies.
In our experimental tests, the proposed extremum seeking controller converged within 2 minutes in the majority of cases.
We think this would be a practically useful convergence time considering the flight time of most multicopters are between 10 to 20 minutes \cite{karydis2017energetics}.
\subsection{Cost of extremum seeking}
Since the perturbations are applied by the extremum seeking controller, the power consumption of the vehicle will be higher than the flight at a constant reference  without perturbations.
In this subsection, we compare the optimal values of the cost function without perturbation (i.e., optimal cost values in \figref{fig:groundTruth}) with the average cost values when flying at the same mean speed and sideslip but with perturbations applied. 
The increases in cost are summarized in Table \ref{tab:cost}.

\begin{table}[ht]
\renewcommand{\arraystretch}{1.2}
\caption{Optimal cost increase due to perturbation}
\vspace{-1ex}
\label{tab:cost}
\centering
\begin{tabular}{|l|l|l|l|l|}
\hline
\begin{tabular}[c]{@{}l@{}}Optimization\\ goal\end{tabular} & \begin{tabular}[c]{@{}l@{}}Payload\end{tabular} & \begin{tabular}[c]{@{}l@{}}Cost without\\ perturbation\end{tabular} & \begin{tabular}[c]{@{}l@{}}Cost with\\ perturbation\end{tabular} & \begin{tabular}[c]{@{}l@{}}Cost\\ increase\end{tabular} \\ \hline
\multirow{2}{*}{range} & box & 12.8 J/m & 13.2 J/m & 3.1 \% \\ \cline{2-5}
 & none & 11.0 J/m & 11.4 J/m & 4.0 \% \\ \hline
\multirow{2}{*}{endurance} & box & 112.5 W & 116.4 W & 3.5 \% \\ \cline{2-5}
 & none & 101 W & 105.2 W & 4.2 \% \\ \hline
\end{tabular}
\end{table}

In summary, the increase in cost was 3.1 - 4.2 $\%$ because of the perturbations applied by the extremum seeking controller.
This is less than the power consumption reduction when flying at the optimal endurance speed compared to hovering, which is 12.6\% with the box payload and 7\% without it, so the advantage of the proposed method  outweighs its cost.
\par
To reduce the impact of this increase, the extremum seeking controller can be enabled only when there is a model change (e.g., picking up a new payload), and disabled after convergence. 
In addition, decreasing the perturbation magnitude will be helpful for reducing the additional cost of perturbation, but this will also reduce the convergence speed.
One should take these two factors into account when selecting the proper perturbation magnitude.

\subsection{Performance under strong wind disturbances}

We further evaluated the performance of the proposed extremum seeking controller with adaptive step size under strong wind disturbances.
Like the aforementioned experiments, the vehicle was commanded to follow a circular path with a radius of 30 meters at 5 meters in height.
The wind was measured by a Young 81000 anemometer at 20 Hz with 0.01 $\text{m/s}$ resolution, at a height of 2 meters.
The extremum seeking controller's parameters are the same as the experiments under light wind in \secref{sec:compare_windstill}.

\begin{figure}[!htp]
    \centering
    \includegraphics[width =\linewidth]{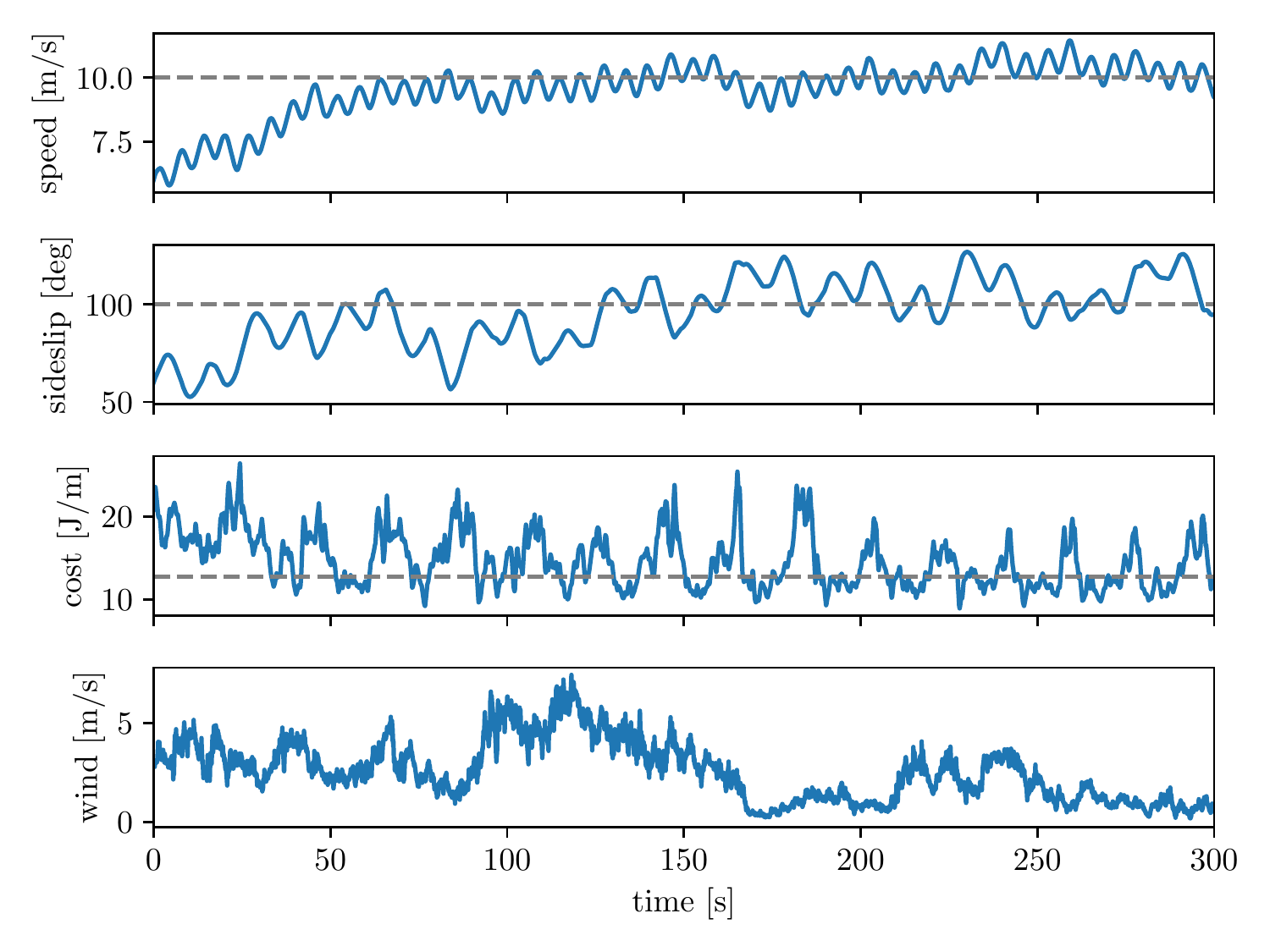}
    \vspace{-1ex}
    \caption{Optimal range speed and sideslip seeking under strong wind disturbances, with the box payload. 
    The optimal values of the speed, sideslip and cost function are marked as grey dashed lines.
    The maximum magnitude of wind disturbances is 7.43 $\text{m/s}$. }
    \label{fig:range_disturbance}
\end{figure}

\begin{figure}[!htp]
    \centering
    \includegraphics[width =\linewidth]{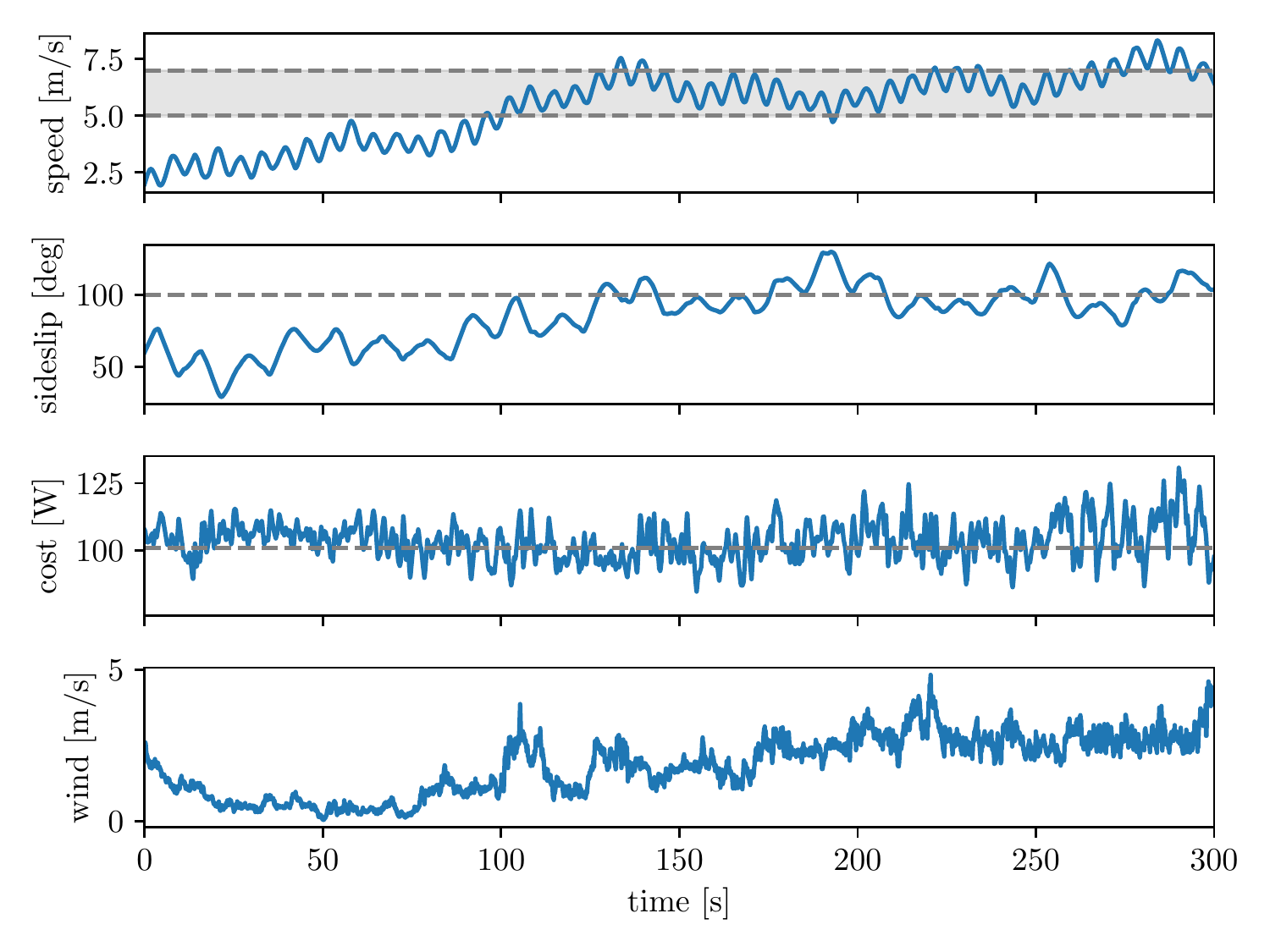}
    \vspace{-1ex}
    \caption{Optimal endurance speed and sideslip seeking under strong wind disturbances, without the box payload. 
    The optimal values of the sideslip and cost function are marked as grey dashed lines.
    The optimal value of the speed is marked as a range between 5 - 7 m/s, because the cost function values are very close in this range, with less than 1\% difference.
    The maximum magnitude of wind disturbances is   4.83 $\text{m/s}$}
    \label{fig:endurance_disturbance}
\end{figure}

The experiments demonstrated that the proposed method was still able to find the optimal range and endurance speed and sideslip, as shown in  \figref{fig:range_disturbance} and \figref{fig:endurance_disturbance}.
The maximum wind magnitude was 7.43 m/s in the optimal range experiment, and was 4.83 m/s in the optimal endurance experiment.
The proposed method is not very sensitive to wind disturbances:
because of the time-scale separation in the extremum seeking controller, the change in the speed and sideslip setpoints by the extremum seeking controller is very slow compared with the closed-loop dynamics of the vehicle.
\par
Compared with the tests with the same initial conditions but under light wind in \secref{sec:convergence_comparison}, the wind disturbances caused larger oscillations in the reference sideslip  (\figref{fig:range_disturbance} compared with the first column of \figref{subfig:escRangeWithBox}) and longer convergence time (\figref{fig:endurance_disturbance} compared with the second column of \figref{subfig:escEnduranceWithoutBox}).

\section{Conclusion}
\label{sec:conclusion}
An online, adaptive, model-free method for finding the speed and sideslip that maximize the flight range or endurance of multicopters is proposed in this work.
Not dependent on any power consumption model of the vehicle, it is able to adapt to different payloads and is easy to deploy. 
The proposed method can mitigate the common problem of limited flight range and endurance of multicopters.
Based on a novel multivariable extremum seeking controller with adaptive step size, it is able to achieve faster convergence compared to the standard extremum seeking controller with fixed step size.
\par
Through realistic outdoor experiments, we show that this method is able to find the optimal speed and sideslip correctly under different payloads and under strong wind disturbances.
In addition to multicopters, this method can also be applied to fixed wing aerial robots to find the optimal flight speed (to achieve the longest flight time or distance) whose sideslip is usually not a free degree of freedom in path planning. 


\section*{Acknowledgements}
This work was partially supported by the J.K. Zee Fellowship, the UC Berkeley Graduate Division Block Grant Award, and the AFRI Competitive Grant no. 2020-67021-32855/project accession no. 1024262 from the USDA National Institute of Food and Agriculture.
The AFRI Competitive Grant is being administered through AIFS: the AI Institute for Next Generation Food Systems. \url{https://aifs.ucdavis.edu}.

Research was also partially sponsored by the Army Research Laboratory and was accomplished under Cooperative Agreement Number W911NF-20-2-0105. 
The views and conclusions contained in this document are those of the authors and should not be interpreted as representing the official policies, either expressed or implied, of the Army Research Laboratory or the U.S. Government. 
The U.S. Government is authorized to reproduce and distribute reprints for Government purposes notwithstanding any copyright notation herein.
\par

{
\bibliographystyle{IEEEtran}
\bibliography{bib/bibliography}
}
\newpage

\label{sec:appendix}
\section*{Appendix: Proof of Proposition 1}
The reduced system \eqref{equ:reduced_model} is in the form where the averaging method is applicable \cite[Chapter 10.4]{khalil2002nonlinear} ($\delta$ is a small positive parameter). Its corresponding averaged system dynamics can be described as follows,
\begin{align}
    \frac{d}{d \tau} \mmat{\tilde{r}^a_v \\ \tilde{r}^a_\beta \\ q^a_v \\ q^a_\beta \\ \tilde{\eta}^a_v \\ \tilde{\eta}^a_\beta \\ m^a_v \\ m^a_\beta} = 
    \delta \mmat{ (-k_v^{'} q^a_v)/\sqrt{m^a_v + \epsilon} \\[3pt]
                  (-k_\beta^{'} q^a_\beta)/\sqrt{m^a_\beta + \epsilon} \\[3pt]
                  \omega_{lv}^{'} \frac{1}{\Pi}\int_0^{\Pi} (v(\bm{\tilde{r}^a} + \bm{\Bar{p}}(\sigma)) \sin \omega_v^{'} \sigma d\sigma - \omega_{lv}^{'}q^a_v \\[3pt]
                  \omega_{l\beta}^{'} \frac{1}{\Pi}\int_0^{\Pi} (v(\bm{\tilde{r}^a} + \bm{\Bar{p}}(\sigma)) \sin \omega_\beta^{'} \sigma d\sigma - \omega_{l\beta}^{'}q^a_\beta \\[3pt]
                  -\omega_{hv}^{'}\tilde{\eta}^a_v + \omega_{hv}^{'} \frac{1}{\Pi} \int_0^{\Pi} v(\bm{\tilde{r}^a} + \bm{\Bar{p}}(\sigma)) d\sigma \\[3pt]
                  -\omega_{h\beta}^{'}\tilde{\eta}^a_\beta + \omega_{h\beta}^{'} \frac{1}{\Pi} \int_0^{\Pi} v(\bm{\tilde{r}^a} + \bm{\Bar{p}}(\sigma)) d\sigma \\[3pt]
                  \gamma_v^{'} (-m^a_v + {q_v^{a}}^2) \\[3pt]
                  \gamma_\beta^{'} (-m^a_\beta + {q_\beta^{a}}^2)
                }, \label{equ:reduced}
\end{align}
where the superscript $a$ denotes the variables of the averaged system, and $\Pi$ is the least common period of sinusoidal functions with frequencies of $\omega_v^{'}$ and $ \omega_\beta^{'}$.

The equilibrium point of the averaged system \eqref{equ:reduced} is denoted as $[\tilde{r}^{a,e}_v, \tilde{r}^{a,e}_\beta, q^{a,e}_v, q^{a,e}_\beta, \tilde{\eta}^{a,e}_v, \tilde{\eta}^{a,e}_\beta, m^{a,e}_v, m^{a,e}_\beta]^T$ which satisfies:
\begin{align}
    & q^{a,e}_v = q^{a,e}_\beta = 0, \\
    & m^{a,e}_v = m^{a,e}_\beta = 0, \\
    & \int_0^{\Pi} (v(\bm{\tilde{r}^{a,e}} + \bm{\Bar{p}}(\sigma)) \sin \omega_v^{'} \sigma d\sigma = 0, \label{equ: r_v_equilibrium}\\
    & \int_0^{\Pi} (v(\bm{\tilde{r}^{a,e}} + \bm{\Bar{p}}(\sigma)) \sin \omega_\beta^{'} \sigma d\sigma = 0, 
    \label{equ: r_s_equilibrium}\\
    & \tilde{\eta}^{a,e}_v = \tilde{\eta}^{a,e}_\beta = \frac{1}{\Pi} \int_0^{\Pi} v(\bm{\tilde{r}^{a,e}} + \bm{\Bar{p}}(\sigma)) d\sigma, \label{equ: eta_equilibrium}
\end{align}
where the superscript $e$ denotes the variables for the equilibrium point.
We consider $\tilde{r}^{a,e}_v $ and $ \tilde{r}^{a,e}_\beta$ as perturbations with second-order Taylor series expansion over $a_v$ and $a_\beta$,
\begin{align}
    \tilde{r}^{a,e}_v =& \ b_{1,v} a_v + b_{2,v}a_\beta \notag \\
    & + b_{3,v} a_v^2 + b_{4,v} a_v a_\beta + b_{5,v} a_\beta^2 + O(\mnorm{\bm{a}}^3), \label{equ: postulate_r_v} \\
    \tilde{r}^{a,e}_\beta =& \ b_{1,\beta} a_v + b_{2,\beta}a_\beta \notag \\
    & + b_{3,\beta} a_v^2 + b_{4,\beta} a_v a_\beta + b_{5,\beta} a_\beta^2 + O(\mnorm{\bm{a}}^3), \label{equ: postulate_r_s}
\end{align}
where $b_{i,v}$ and $b_{i,\beta}$ ($i = 1,..,5$) are constant numbers. By substituting $\eqref{equ: postulate_r_v}$, $\eqref{equ: postulate_r_s}$ into $\eqref{equ: r_v_equilibrium}$, $\eqref{equ: r_s_equilibrium}$, integrating and equating the like powers of $a_v$ and $a_\beta$, we can find that the first-order coefficients and second-order coefficients for the mixing terms are zero, and $\tilde{r}^{a,e}_v$ and $\tilde{r}^{a,e}_\beta$ can be written as:
\begin{align}
    & \tilde{r}^{a,e}_v = b_{3,v} a_v^2 + b_{5,v} a_\beta^2 + O(\mnorm{\bm{a}}^3), \label{equ: r_v_equilibrium_result}\\
    & \tilde{r}^{a,e}_\beta = b_{3,\beta} a_v^2 + b_{5,\beta} a_\beta^2 + O(\mnorm{\bm{a}}^3). \label{equ: r_s_equilibrium_result}
\end{align}
In addition, by substituting $\eqref{equ: r_v_equilibrium_result}$, $\eqref{equ: r_s_equilibrium_result}$ into $\eqref{equ: eta_equilibrium}$ and integrating, we can get 
\begin{align}
    & \tilde{\eta}^{a,e}_v = \tilde{\eta}^{a,e}_\beta = \frac{1}{4} (H_{11} a_v^2 + H_{22} a_\beta^2) + O(\mnorm{\bm{a}}^3).
\end{align}

At the equilibrium point of the averaged system in \eqref{equ:reduced}, the Hessian $J^{a,e}_r$ is a block-diagonal matrix as follows,
\begin{equation}
\label{equ:jacobian}
    J^{a,e}_r = \delta \begin{bmatrix}
    A & 0_{4 \times 4} \\
    B & -\text{diag}(\omega^{'}_{hv}, \omega^{'}_{h\beta}, \gamma^{'}_v, \gamma^{'}_\beta)
    \end{bmatrix},
\end{equation}
where $A, B \in \mathbb{R}^{4 \times 4}$,
\begin{equation}
    A = \begin{bmatrix}
    0 & 0 & -k^{'}_v/\sqrt{\epsilon} & 0 \\
    0 & 0 & 0 & -k^{'}_\beta/\sqrt{\epsilon} \\
    A_{31} & A_{32} & -\omega^{'}_{lv} & 0 \\
    A_{41} & A_{42} & 0 & -\omega^{'}_{l\beta}
    \end{bmatrix},
\end{equation}
\begin{equation}
    B = \begin{bmatrix}
    B_{11} & B_{12} & 0 & 0 \\
    B_{21} & B_{22} & 0 & 0 \\
    0 & 0 & 0 & 0 \\
    0 & 0 & 0 & 0
    \end{bmatrix},
\end{equation}
with expressions of two matrices,
\begin{align}
    & \mmat{A_{31} & A_{32}}^T = \frac{\omega^{'}_{lv}}{\Pi} \int_0^\Pi \mpd{v(\tilde{\bm{r}}^{a,e}+\bm{\Bar{p}}(\sigma))}{\tilde{\bm{r}}^{a,e}} \sin \omega^{'}_v \sigma d \sigma, \\[3pt]
    & \mmat{A_{41} & A_{42}}^T = \frac{\omega^{'}_{l\beta}}{\Pi} \int_0^\Pi \mpd{v(\tilde{\bm{r}}^{a,e}+\bm{\Bar{p}}(\sigma))}{\tilde{\bm{r}}^{a,e}} \sin \omega^{'}_\beta \sigma d \sigma, \\
    & \mmat{B_{11} & B_{12}}^T = \frac{\omega^{'}_{hv}}{\Pi} \int^{\Pi}_0 \mpd{v(\tilde{\bm{r}}^{a,e}+\bm{\Bar{p}}(\sigma))}{\tilde{\bm{r}}^{a,e}} d \sigma, \\
    & \mmat{B_{21} & B_{22}}^T = \frac{\omega^{'}_{h\beta}}{\Pi} \int^{\Pi}_0 \mpd{v(\tilde{\bm{r}}^{a,e}+\bm{\Bar{p}}(\sigma))}{\tilde{\bm{r}}^{a,e}} d \sigma.
\end{align}
Hence, the block-lower-triangular matrix $J^{a,e}_r$ in \eqref{equ:jacobian} is Hurwitz if and only if that all diagonal submatrices are Hurwitz. Since $\delta,  \gamma^{'}_v$, $\gamma^{'}_\beta$, $\omega^{'}_{hv}$ and $\omega^{'}_{h\beta}$ are positive constants, it remains to prove $A$ as Hurwitz for stability.

With a first-order Taylor expansion we can get that 
\begin{align}
    \mmat{A_{31} & A_{32} \\ A_{41} & A_{42}} = \frac{1}{2} \mmat{\omega^{'}_{lv} a_v & 0 \\ 0 & \omega^{'}_{l\beta} a_\beta} H + O(\mnorm{\bm{a}}). 
\end{align}
The characteristic polynomial of $A$ with roots $\lambda$ can be written by computing the determinant of $\lambda I - A$, 
\begin{align}
    & \text{det} (\lambda I -  A) \nonumber \\
    & = \text{det} \left(  \lambda I \left( \lambda I + \delta \mmat{\omega^{'}_{lv} & 0 \\ 0 & \omega^{'}_{l\beta}} \right) \right. \nonumber \\ 
    & \left. \quad + \frac{\delta^2}{\sqrt{\epsilon}} \mmat{A_{31} & A_{32} \\ A_{41} & A_{42}} \mmat{k_v^{'} & 0 \\ 0 & k_\beta^{'} } \right) \notag \\
    & = \text{det} \left( \lambda^2 I + \lambda \delta \mmat{\omega^{'}_{lv} & 0 \\ 0 & \omega^{'}_{l\beta}} \right. \nonumber \\
    & \left. \quad + \frac{\delta^2}{2\sqrt{\epsilon}} \mmat{\omega^{'}_{lv} a_v & 0 \\ 0 & \omega^{'}_{l\beta} a_\beta} H \mmat{k_v^{'} & 0 \\ 0 & k_\beta^{'} } + O(\delta^2 \mnorm{\bm{a}}) \right), \label{equ:det}
\end{align}
which can be expanded to a 4th order polynomial of $\lambda$.
Under the assumptions that $\mnorm{\bm{a}}$ is small and that the Hessian $H$ in \eqref{minimum-perturbation-dynamics} is positive,
the roots of this 4th order polynomial can be shown have negative real parts using the Routh-Hurwitz criterion \cite[Chap. 6.2]{nise2015control}, implying that $A$ is Hurwitz. Therefore, $J^{a,e}_r$ is proven as Hurwitz.
The Hurwitz Jacobian $J^{a,e}_r$ indicates that the equilibrium point of the averaged system \eqref{equ:reduced} is locally exponentially stable if $a_v$ and $a_\beta$ are sufficiently small.
Then according to \cite[chapter 10.4]{khalil2002nonlinear}, the theorem is proved. 

\begin{IEEEbiography}
    [{\includegraphics[width=1in,height=1.25in,clip,keepaspectratio]{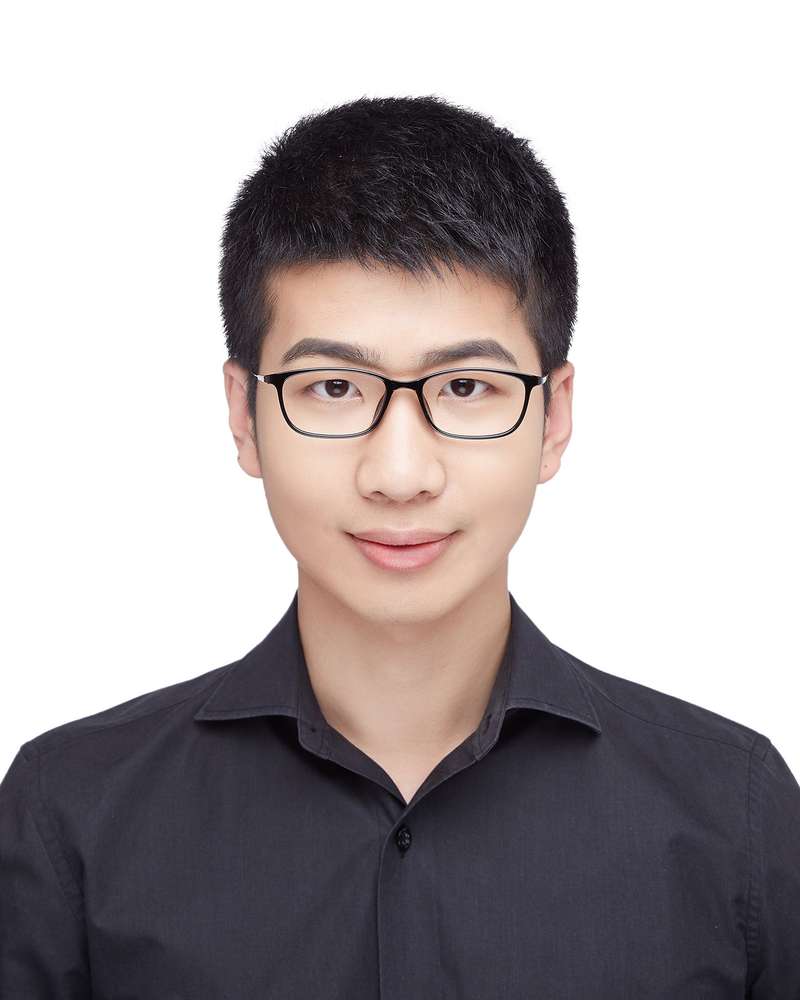}}]{Xiangyu Wu}
received his bachelor of science degree from Beijing Institute of Technology, China in 2017 and master of science degree from University of California, Berkeley, USA in 2019.
He is currently a Ph.D. candidate at the High Performance Robotics Lab at UC Berkeley.
His current research interests are the state estimation and path planning of multicopters.
\end{IEEEbiography}
\begin{IEEEbiography}
    [{\includegraphics[width=1in,height=1.25in,clip,keepaspectratio]{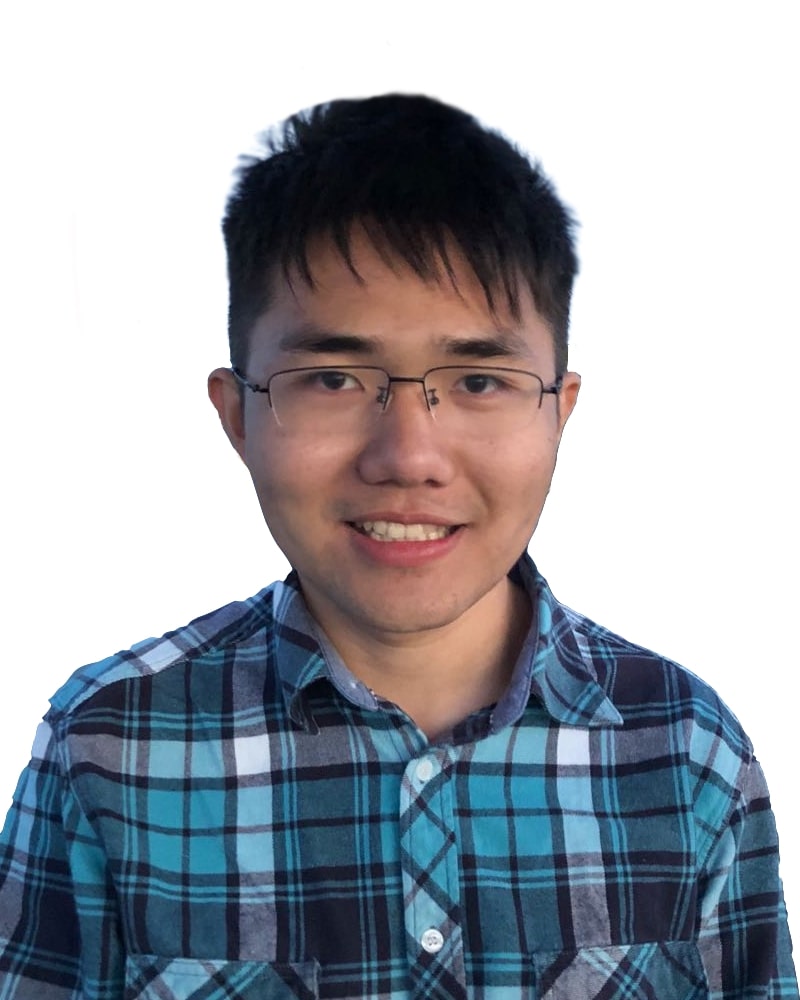}}]{Jun Zeng}
received his B.S.E degree from Shanghai Jiao Tong University (SJTU), China in 2016 and Dipl. Ing. from Ecole Polytechnique, France in 2017. He is currently a Ph.D. candidate supervised by Koushil Sreenath at Hybrid Robotics Group of Mechanical Engineering at University of California, Berkeley, USA. His research interests lie at the intersection of control, optimization and learning on robotics.
\end{IEEEbiography}
\begin{IEEEbiography}
    [{\includegraphics[width=1in,height=1.25in,clip,keepaspectratio]{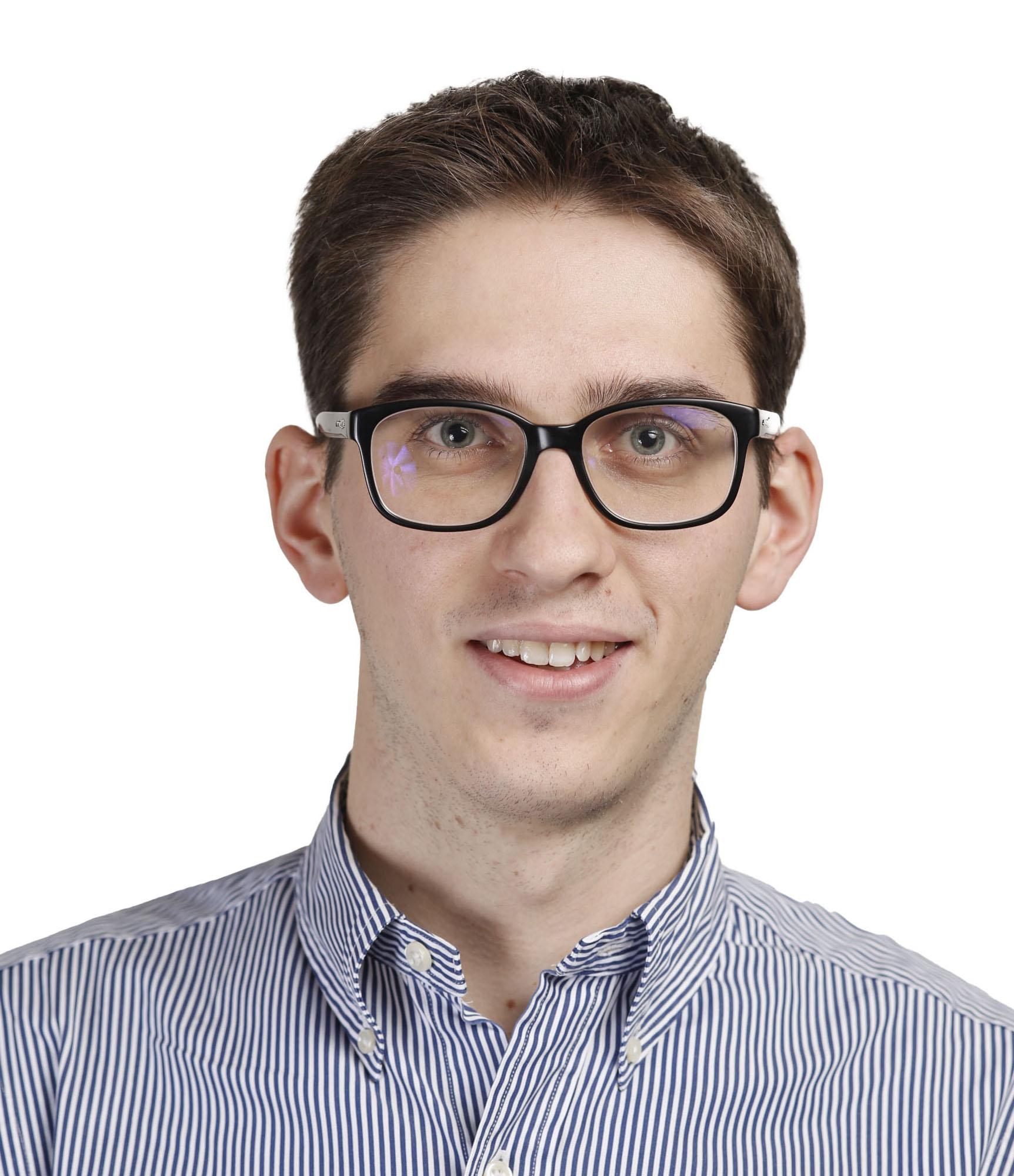}}]{Andrea Tagliabue}
Andrea  Tagliabue is a Ph.D. candidate with the Laboratory for Information and Decision Systems at MIT. Prior to that, he was a Robotics Engineer at Caltech, affiliate with NASA’s Jet Propulsion  Laboratory, and visiting researcher at U.C. Berkeley. He received a M.Sc. from  ETH Zurich, and a B.S. with honours from Politecnico di Milano. His research interests include learning-based methods for control, planning and state estimation for aerial robots.
\end{IEEEbiography}
\begin{IEEEbiography}
    [{\includegraphics[width=1in,height=1.25in,clip,keepaspectratio]{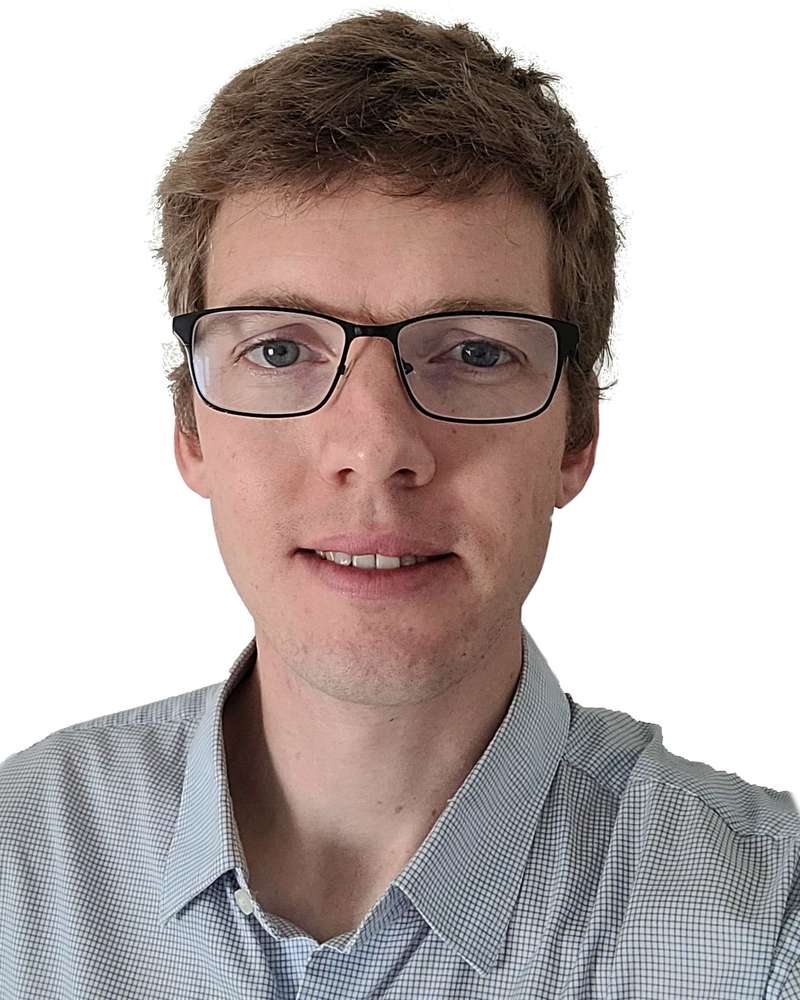}}]{Mark W. Mueller}
Mark w. Mueller is an assistant professor of Mechanical Engineering at the University of California, Berkeley, and runs the High Performance Robotics Laboratory (HiPeRLab). He received a Dr.Sc. and M.Sc. from the ETH Zurich in 2015 and 2011, respectively, and a BSc from the University of Pretoria in 2008. His research interests include aerial robotics, their design and control, and especially the interactions between physical design and algorithms.
\end{IEEEbiography}

\end{document}